\documentclass[sigconf]{acmart}

\def\equationautorefname~#1\null{Equation~(#1)\null}
\usepackage{breakurl}
\usepackage{color}
\usepackage{bm}
\usepackage{xspace}
\usepackage{algorithm}
\usepackage{algorithmic}

\usepackage{mdwlist}    
\usepackage{paralist} 

\usepackage{utfsym}

\usepackage{adjustbox}
\usepackage{array}
\usepackage{booktabs}
\usepackage{multirow}

\usepackage{array,graphicx}
\usepackage{booktabs}
\usepackage{pifont}
\usepackage{color}

\usepackage{colortbl}
\definecolor{lightgray}{gray}{0.85}
\usepackage{multirow} 

\usepackage{fancybox} 
\usepackage{amsmath}
\usepackage{amsfonts}

\DeclareMathOperator{\Tr}{Tr}
\DeclareMathOperator{\sds}{S_{++}^{p}}

\usepackage{amsthm}

\newcommand{\RomanNumeralCaps}[1]
    {\MakeUppercase{\romannumeral #1}}

\usepackage{xcolor}
\usepackage{wasysym}

\usepackage{cleveref}

\newcommand{\TSK}[1]{#1}

\newcommand{\bit}{\vspace{-0em}\begin{itemize}}
\newcommand{\bitnone}{\vspace{-0em}\begin{itemize}[label={}]}
\newcommand{\eit}{\end{itemize}\vspace{-0.2em}}
\newcommand{\ben}{\vspace{-0em}\begin{enumerate}}
\newcommand{\een}{\end{enumerate}\vspace{-0.2em}}
\newcommand{\bea}{\vspace{-0em}\begin{eqnarray}}
\newcommand{\eea}{\end{eqnarray}\vspace{-0.0em}}
\newcommand{\beq}{\vspace{-0.0em}\begin{equation}}
\newcommand{\eeq}{\end{equation}\vspace{-0.0em}}

\newtheorem{informalProblem}{Informal Problem}
\newtheorem{problem}{Problem}

\newtheorem{lemma}{\textsc{Lemma}}
\newtheorem{customlemma}{\textsc{Lemma}}
\newtheorem{definition}{Definition}

\newcommand{\hide}[1]{}
\newcommand{\reminder}[1]{#1}
\newcommand{\reminderrm}[1]{}

\newcommand{\myparaitemize}[1]{\noindent{\textbf{#1.}}}

\newcommand{\eq}[1]{Eq.~(#1)}

\newcommand{\fig}[1]{Figure~#1}

\newcommand{\appen}[1]{Appendix~#1}

\newcommand{\method}{\textsc{TimeCast}\xspace}
\newcommand{\methodmodel}{\textsc{TimeCast}\xspace}

\newcommand{\subject}{instance\xspace}
\newcommand{\subjects}{instances\xspace}

\newcommand{\intera}{dependency\xspace}

\newcommand{\argmax}{\operatornamewithlimits{argmax}}

\newcommand{\mathR}{\mathbb{R}}

\newcommand{\vecx}{x}
\newcommand{\matx}{X}
\newcommand{\ntime}{T}
\newcommand{\ltime}{t}

\newcommand{\ldim}{d}
\newcommand{\nseq}{V}
\newcommand{\lseq}{v}
\newcommand{\llstage}{k'}
\newcommand{\lstage}{k}
\newcommand{\nstage}{K}

\newcommand{\newlseq}{w}
\newcommand{\newnseq}{W}
\newcommand{\ctime}{t_c}

\newcommand{\alllabeleddata}{\mathcal{D}}

\newcommand{\rul}{\tau}

\newcommand{\indep}{\perp \!\!\! \perp}
\newcommand{\Func}{\mathcal{F}}
\newcommand{\allmodel}{\Theta}
\newcommand{\model}{\theta}

\newcommand{\allstageas}{S}
\newcommand{\stageas}{s}

\newcommand{\dmean}{\mu}
\newcommand{\dprecision}{\Lambda}
\newcommand{\empcov}{Q}

\newcommand{\plevel}{W}
\newcommand{\drift}{\nu}
\newcommand{\bstd}{\sigma_B}
\newcommand{\pfunc}{f}

\newcommand{\wtime}{\tau}

\newcommand{\lassopara}{\alpha}
\newcommand{\predpara}{\beta}
\newcommand{\Obj}{\Psi}
\newcommand{\obj}{\psi}
\newcommand{\weight}{A}
\newcommand{\cost}[2]{\psi_{s}(#1~|~#2)}
\newcommand{\optcost}[2]{C_{#1,#2}}

\newcommand{\fset}{\mathcal{U}}
\newcommand{\iter}{r}
\newcommand{\bv}{u}
\newcommand{\pramid}{i}
\newcommand{\pramidsub}{j}
\newcommand{\pidd}{d}
\newcommand{\pidp}{p}
\newcommand{\pids}{\allstageas}
\newcommand{\lp}{z}

\newcommand{\conviter}{\hat{r}}

\newcommand{\window}{m}

\newcommand{\cmapss}{\textit{Engine}\xspace}
\newcommand{\mapm}{\textit{Factory}\xspace}
\newcommand{\chronic}{\textit{ICU-Chronic}\xspace}
\newcommand{\acutedata}{\textit{ICU-Acute}\xspace}
\newcommand{\mixed}{\textit{ICU-Mixed}\xspace}

\renewcommand\footnotetextcopyrightpermission[1]{}

\acmConference[KDD '26]{Proceedings of the 32nd ACM SIGKDD Conference on Knowledge Discovery and Data Mining V.1}{August 09--13, 2026}{Jeju Island, Republic of Korea}
\acmBooktitle{Proceedings of the 32nd ACM SIGKDD Conference on Knowledge Discovery and Data Mining V.1 (KDD '26), August 09--13, 2026, Jeju Island, Republic of Korea}
\acmDOI{10.1145/3770854.3780164}

\author{Kota Nakamura}
\affiliation{%
  \institution{InfoTech, Toyota Motor Corporation}
  \state{Tokyo}
  \country{Japan}
}
\email{kota.nakamura.10@toyota.global}
\author{Koki Kawabata}
\affiliation{%
  \institution{SANKEN, Osaka University}
  \state{Osaka}
  \country{Japan}
}
\email{koki@sanken.osaka-u.ac.jp}
\author{Yasuko Matsubara}
\affiliation{%
  \institution{SANKEN, Osaka University}
  \state{Osaka}
  \country{Japan}
}
\email{yasuko@sanken.osaka-u.ac.jp}
\author{Yasushi Sakurai}
\affiliation{%
  \institution{SANKEN, Osaka University}
  \state{Osaka}
  \country{Japan}
}
\email{yasushi@sanken.osaka-u.ac.jp}
\begin{document}
\title{Fast Mining and Dynamic 
Time-to-Event Prediction over Multi-sensor Data Streams
}
\begin{abstract}
    Given real-time sensor data streams obtained from machines, how can we continuously predict when a machine failure will occur? This work aims to continuously forecast the timing of future events by analyzing multi-sensor data streams. A key characteristic of real-world data streams is their dynamic nature, where the underlying patterns evolve over time. To address this, we present \textsc{TimeCast}, a dynamic prediction framework designed to adapt to these changes and provide accurate, real-time predictions of future event time. Our proposed method has the following properties: (a)~\textit{Dynamic:} it identifies the distinct time-evolving patterns (i.e., stages) and learns individual models for each, enabling us to make adaptive predictions based on pattern shifts. (b)~\textit{Practical:} it finds meaningful stages that capture time-varying interdependencies between multiple sensors and improve prediction performance; (c)~\textit{Scalable:} our algorithm scales linearly with the input size and enables online model updates on data streams. Extensive experiments on real datasets demonstrate that \textsc{TimeCast} provides higher prediction accuracy than state-of-the-art methods while finding dynamic changes in data streams with a great reduction in computational time.
\end{abstract}

\maketitle
\section{Introduction}
    \label{010intro}
    With the rapid growth in Internet of Things (IoT) deployment,
real-time sensor data is being generated and collected by 
a wide range of applications 
\cite{sakurai2017smart},
including 
automated factories
\cite{gubbi2013internet},
digital twins \cite{tao2018digital,matsubara2025microadapt},
and electronic health record systems \cite{johnson2016mimic,kotogeevobrain},
from which 
one of the most fundamental demands in data science and engineering is 
deriving actionable insights,
such as predicting the timing of future machine failures or patient deaths.
For example, 
a significant interest for industrial managers 
is obtaining more accurate estimates of failure time 
to schedule preventive maintenance that minimizes downtime and maximizes operational lifetime
\cite{juodelyte2022predicting}.
For patient monitoring in intensive care units (ICUs),
it is essential to continuously estimate
the time (i.e., risk) of a clinically critical event,
such as death or the onset of disease, 
for better hospital resource management 
by focusing on patients that need it most 
\cite{harutyunyan2019multitask}.
To address these scenarios, 
we focus on 
an important yet challenging problem, 
namely,
streaming time-to-event prediction,
where 
our goal is to analyze real-time sensor sequences
and continuously predict when a future event will occur.

Time-to-event prediction 
captures relationships between 
observations 
(e.g., sensor readings) and 
the time duration until an event of interest occurs.
It can predict the event probabilities as a function of time, 
allowing us to flexibly 
assess the risk of event occurrence at any given time.
In contrast, widely used binary classifiers predict 
whether the event of interest occurs after a predetermined duration (e.g., $30$ seconds) and 
can only assess risk at a specific time
\cite{chen2025long,kotoge2024splitsee}.

The problem of time-to-event prediction becomes more challenging 
when data arrives in a streaming or online manner. 
Assume that we have a sensor data stream $\matx_{:\ctime}$, 
which is a time-evolving sequence of $\ldim$-dimensional observations,
i.e., $\matx_{:\ctime} = \{\vecx_{1},\ldots, \vecx_{\ctime} \}$, 
where $x_{\ctime}$ is the most recent observation, 
and $\ctime$ increases with every new time tick.
Such a situation requires an efficient algorithm that 
analyzes the continuously growing data stream
and makes real-time predictions 
to design countermeasures as soon as the risk increases.
Moreover, sensor data streams are usually non-stationary, 
changing their behavior over time.
For example, in ICU patient monitoring,
sensor measurements of vital signs 
change through distinct temporal phases, 
reflecting the patient’s condition 
as it approaches clinically critical events \cite{yoon2016discovery}.
These phases are key features that represent 
the temporal evolutions for entire data streams, 
which we specifically refer to as ``\textit{stages}'' hereafter.
Unlike existing time-to-event prediction approaches
\cite{wang2019machine,li2016multi,ameri2016survival,lee2018deephit,kvamme2019time},
which are static (as opposed to dynamic) and 
seek to predict the event time 
based on
individual observations $\vecx_{\ctime}$,
the ideal method should model 
time-evolving stages as the sequence-level features of $\matx_{:\ctime}$.
\textit{So what are meaningful stages 
for time-to-event prediction?}
We aim to find distinct temporal patterns that not only capture 
latent structural similarities in observations but also enhance prediction performance.

In this paper, we present \method,
a dynamic approach for time-to-event prediction
over multi-sensor data streams.
\method is based on a sequential multi-model 
structure that identifies meaningful stages in data streams by jointly learning descriptive and predictive features.
Thus, it can effectively predict event probabilities at future time points
while adapting to stage shifts.
In short, the problem we wish to solve is as follows.
\begin{informalProblem}
\textbf {Given} 
a sensor data stream $X_{:\ctime}$ for a machine/patient 
at risk of an event of interest occurring,
which consists of
observations until the current time $\ctime$,
i.e., $X_{:\ctime} = \{x_{1},\ldots, x_{\ctime}
\}$,
\bit
\item
\textbf {Find} time-evolving stages that improve prediction performance 
while identifying latent structural similarities in $X_{:\ctime}$,
\item 
\textbf {Predict} 
event probabilities at future time points,
\eit
continuously and quickly in a streaming fashion. 
\end{informalProblem}
\myparaitemize{Preview of Our Results}
\begin{figure}[t]
     \centering
     \hspace{-1.7em}
    \includegraphics[width=1.05\columnwidth]{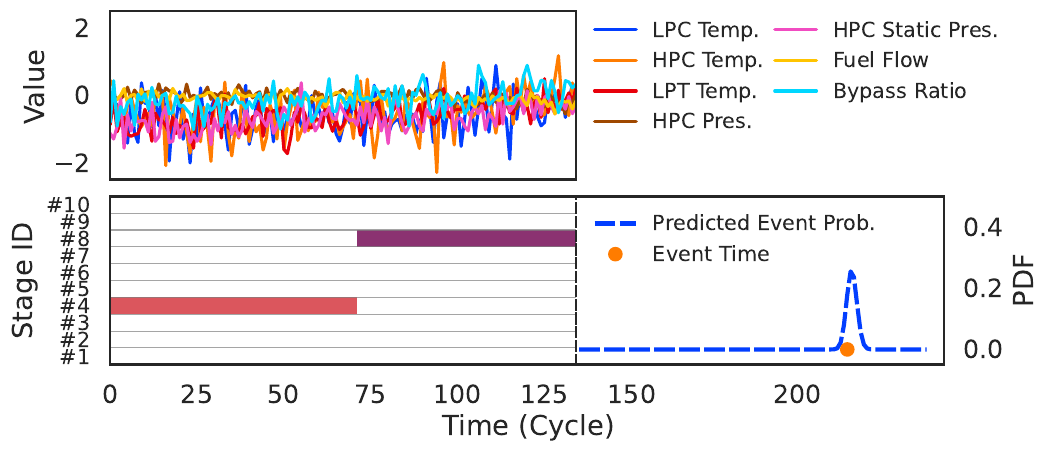}\\
    \vspace{-0.5em}
    (a) Snapshots of streaming time-to-event prediction ($\ctime=135$)\\
     \hspace{-1.7em}
    \includegraphics[width=1.05\columnwidth]{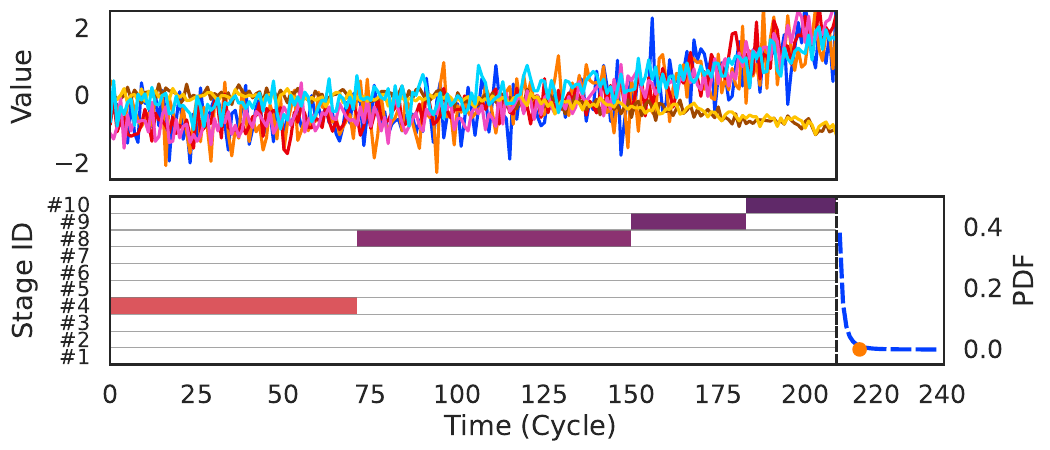}\\
    \vspace{-0.5em}
    (b) Snapshots of streaming time-to-event prediction ($\ctime=210$)
    \vspace{-1em}
\caption{Prediction results of \method over 
a machine failure-related sensor data stream.
The method continuously detects/updates time-evolving stages.
Then, it adaptively predicts event probabilities depending 
on the current stages.}
\label{fig:preview}
\end{figure}
Figure~\ref{fig:preview} shows an example of \method 
applied to turbofan jet engine data.
The data consists of seven sensor readings,
including temperatures and pressures, measured at every cycle.
Given the real-time sensor data stream,
\method continuously provides predictions for the time to failure
while capturing stages and their changes behind the data stream.

Figure~\ref{fig:preview}~(a) shows 
the original data stream (top) and 
the result we obtained with \method (bottom) 
at the current time $\ctime=135$.
\method firstly identifies stages and their changes 
in the sensor data stream observed up to the current time $\ctime$.
The figure illustrates that 
\method detects a stage shift from Stage~\#4 to Stage~\#8 around the time $\ltime=75$ and recognizes Stage~\#8 as the current stage.
Note that the original data does not exhibit any obvious patterns or stages.
Meanwhile, our method finds the stages that 
contribute to prediction performance while revealing interdependencies between sensors, 
as elaborated in Section~\ref{030problem}.
%
Finally, 
\method predicts the probabilities of a failure event 
by employing a stochastic time-to-event predictor associated 
with the current stage.
The bottom part of Figure~\ref{fig:preview}~(a) 
also shows the prediction results,
where the dashed blue line represents 
the predicted event probabilities, 
and the orange dot indicates the actual time that a failure occurs.
The result demonstrates that 
\method accurately predicts the failure time,
i.e., the predicted probabilities show
a high value around the actual event time.

Figure~\ref{fig:preview}~(b) shows 
the snapshots of \method outputs at the current time $\ctime=210$,
where \method identifies the stage shifts 
(i.e., $\#4 \rightarrow \#8 \rightarrow \#9 \rightarrow \#10$),
and then adaptively makes predictions 
with the predictor for Stage~$\#10$.
Here, the predicted event probabilities provide
a relatively high value at the recent time,
indicating a significant risk and suggesting the need for 
immediate shutdown or maintenance.
Consequently,
the method continuously detects/updates the shifting points 
in the sensor data streams 
and adaptively predicts the event time 
while switching predictors depending on the stages.
As we will show in the experiments,
our dynamic prediction approach improves prediction accuracy
with a great reduction in prediction time.

\myparaitemize{Contribution} 
The main contributions of our paper are:
\bit
    \item  
    \textit{Dynamic prediction approach}:
    We propose a novel prediction approach, \method,
    which captures stages behind non-stationary sequences 
    and adaptively predicts the event probabilities at future time points.
    \item
    \textit{Practicality}:
    By jointly learning descriptive and predictive features, \method can make accurate predictions while revealing individual temporal patterns and time-varying interdependencies between sensors.
    \item 
    \textit{Scalability}:
    The computational time of our algorithm is linear in the data size, with fast convergence. It can process incoming data in an online manner.
\eit

\myparaitemize{Reproducibility}
Our source code and datasets are available at \cite{WEBSITE}.

\myparaitemize{Outline}
The rest of this paper is organized in a conventional way. 
After introducing related studies in Section 2, 
we formally define our problems and present our model in Section 3.
We then propose the algorithms in Section 4. 
We provide our experimental results in Section 5, followed by a conclusion in Section 6.
%

\section{Related Work}
    \label{020related}
    \TSK{
\newcommand*\rot{\rotatebox{90}}

\newcommand*\OK{\ding{51}}
\newcommand*\NO{}
\newcommand{\SOME}{\small\rotatebox{0}{some}}

\TSK{
\begin{table}[t]
\centering
\caption{Capabilities of approaches.
}
\label{table:capability}
\vspace{-0.5em}
\vspace{-0.5em}
\scalebox{0.82}{
\begin{tabular}{l||c|cc|cc|c}
\toprule
&
\rot{DeepSurv/++} &
\rot{HMM/++} &
\rot{TS2Vec} &
\rot{CubeScope} & 
\rot{AC-TPC} &
\rot{\method} \\
\midrule
Time-to-Event Prediction  &\OK &\NO &\SOME &\NO &\OK &\OK \\
Time-Series Modeling  &\NO &\OK &\OK &\OK &\NO &\OK\\
Non-Stationarity &\NO &\OK &\NO &\OK &\OK &\OK\\
Predictive Clustering      &\NO &\NO &\NO &\NO &\OK &\OK\\
Streaming Time-to-Event Prediction &\NO &\NO &\NO &\NO &\NO &\OK\\
\bottomrule
\end{tabular}
}
\end{table}
}



}
The mining of time-stamped event data has attracted great interest 
in many fields \cite{MatsubaraSFIY12,kawabata2021ssmf,bhatia2023sketch,kamarthi2022camul,jang2023fast,li2020data,higashiguchi2025d,chihara2025modeling,fujiwara2025modeling}.
Table~\ref{table:capability} summarizes the relative advantages of our method
in relation to
five aspects, and only \method meets all the requirements.
Our work lies at the intersection of the following three categories.

\myparaitemize{Time-to-Event Prediction}
Event prediction methods based on temporal point processes \cite{shchur2021neural,bosser2023predictive}, 
such as Hawkes process \cite{dubey2023time,zhang2020self}
and cascade Poisson process \cite{10.5555/3023549.3023614,iwata2013discovering},
can 
model dependencies between recurrent events,
where they account for how the occurrence of past events influences the probability of future events while capturing the nature of events over time.
Differing from these methods, 
time-to-event prediction (also called survival analysis)
\cite{wang2019machine,li2016multi,ameri2016survival}
models the relationships between observations (e.g., sensor readings) 
and the remaining time until an event occurs.
%
These methods map observations to parameters of a 
stochastic process for the event time,
such as a Wiener process \cite{zhang2018degradation,lee2010proportional}.
%
Recent works have extended the classical Cox proportional hazards model \cite{cox1972regression} 
with neural networks \cite{katzman2018deepsurv,lee2018deephit,ren2019deep,xue2020deep}
to capture nonlinear relationships.
Cox-Time \cite{kvamme2019time} relaxes 
the proportionality assumption of the Cox model,
improving flexibility for large-scale data sets.
DeepHit \cite{lee2018deephit} is capable of capturing 
multiple types of events and their completing risks.
%
However, 
existing methods are static and 
are not intended to handle streams of time-varying observations.
In contrast, our method is a dynamic prediction approach that
can be aware of changes in data streams by incorporating time series modeling.

\myparaitemize{Time Series Modeling}
Hidden Markov models (HMM)
and other dynamic statistical models 
are extended to capture non-stationary sequences
and dynamically changing trends,
known as \textit{concept drift} \cite{lu2018learning},
by performing 
the simultaneous segmentation and clustering of the time series
\cite{MatsubaraSF14,HooiLSF17,hooi2019branch}.
To identify the segments and clusters effectively,
Gaussian graphical models (GGMs) and their variants 
\cite{hallac2017toeplitz,tomasi2018latent,tozzo2021statistical,obata2024mining}
capture the 
interdependencies of variables in each subsequence.
StreamScope \cite{kawabata2019automatic} extends 
a hierarchical HMM-based model to analyze data streams.
DMM \cite{obata2024dynamic} is a GGM-based subsequence clustering method 
that can identify segments and clusters across multiple sequences.
Although these approaches find segments and clusters 
by focusing only on the similarity of observations,
our proposed method identifies them 
by simultaneously evaluating both similarity and prediction performance.
Deep neural network (DNN) models, 
including representation learning methods \cite{yue2022ts2vec,lei2019similarity,cheng2020time2graph} 
and 
transformer-based architectures
\cite{yu2023dsformer,Yuqietal-2023-PatchTST},
provide
end-to-end learning frameworks
to capture the dynamics of sequences.
However, these methods are not designed to update models according to time-evolving data streams.

\myparaitemize{Summarization and Clustering}
Probabilistic generative models
\cite{MatsubaraSFIY12,beutel2014cobafi,nakamura2025cybercscope}
have been used to analyze large-scale event data, 
from which they find meaningful clusters, 
such as progression stages \cite{yang2014finding,shin2018discovering}.
CubeScope \cite{nakamura2023fast} has the ability to 
summarize time-stamped event streams and capture distinct temporal clusters.
However, these methods 
are incapable of modeling the relationships between clusters and future events
or predicting the time to event.
Predictive clustering \cite{blockeel2019predictive,qin2023t,aguiar2022learning} 
is a powerful technique 
for combining predictions on future outcomes with clustering.
AC-TPC \cite{lee2020temporal} is a
deep learning approach 
for temporal predictive clustering.
However, it is not designed for time-to-event prediction, 
as it does not incorporate the sequential connectivity of clusters.

Consequently,  
none of these studies focuses on
fast and dynamic time-to-event prediction for 
non-stationary data streams.

\section{Proposed Model}
    \label{030problem}
    \begin{table}[t]
\centering
\footnotesize
\caption{Symbols and definitions.}
\label{table:define}
\vspace{-0.5em}
\vspace{-0.5em}
\scalebox{1}{
\begin{tabular}{l|l}
\toprule
Symbol & Definition \\
\midrule
$\nseq, \newnseq$ & Number of \subjects for learning and prediction process, respectively \\
$X_{\lseq}$ & Sensor sequence for $\lseq$-th \subject, 
i.e., $X_{\lseq}=\{x_{\lseq,1}, \ldots, x_{\lseq,\ntime_{\lseq}} \}$ \\
$x_{\lseq,\ltime}$ &  
$\ltime$-th multivariate observation in $\lseq$-th \subject, 
i.e., $x_{\lseq,\ltime} \in \mathR^{\ldim}$ \\
$\ldim$ & Number of sensor variables \\
$\rul_{\lseq,\ltime}$ & 
Remaining time until an event of interest 
occurs for \subject $\lseq$ at time $\ltime$.
\\
$\alllabeleddata$ & Labeled collection, i.e., 
$\alllabeleddata = 
\{(\matx_{\lseq, :\ltime},\rul_{\lseq,\ltime})\}_{\lseq,\ltime=1}^{\nseq,\ntime_{\lseq}}$ \\
$\matx_{\newlseq,:\ltime}$ 
& Sensor data stream for $\newlseq$-th \subject
i.e., $X_{\newlseq,:\ltime}
=\{x_{\newlseq,1}, \ldots, x_{\newlseq,:\ltime},\ldots\}$ \\
%
\midrule
$\nstage$ & Number of stage models \\
$\model^{(\lstage)}$ & Stage model, 
i.e, $\model^{(\lstage)} = 
\{\dmean^{(\lstage)},\dprecision^{(\lstage)}, \pfunc^{(\lstage)}, \bstd^{(\lstage)} \}$ \\
$\allmodel$ & Stage model set, 
i.e, $\allmodel = \{\model^{(\lstage)}\}_{\lstage=1}^{\nstage}$ \\
$\stageas_{\lseq,\ltime}$ & Stage assignment for observations $\vecx_{\lseq,\ltime}$,
i.e., $\stageas_{\lseq,\ltime} \in \{1,\ldots,\nstage\}$ \\
$\allstageas$ & Stage assignment set \\
$\Func$ &
Full parameter set of \methodmodel,
i.e., $\Func = \{\allmodel,\allstageas\}$ \\
\bottomrule
\end{tabular}
}
\normalsize
\end{table}
In this section, 
we propose our model for streaming time-to-event prediction.
We begin by introducing our formal problem definition, 
and then describe our model in detail.

\subsection{Problem Formulation}
Table~\ref{table:define} lists 
the main symbols that we use throughout this paper.
Let us consider a collection of longitudinal sensor sequences,
where multiple sensor readings are obtained 
from multiple \subjects (e.g., machines or patients),
at every time point, that is, 
each entry is composed of the form \textit{(\subject, sensor, time)}.
Our goal is to (a) learn a prediction model using 
$\nseq$ \subjects for whom the event of interest has occurred, 
and (b) continuously predict the future event time 
for $\newnseq$ \subjects not observed during the learning process,
where $\nseq$ and $\newnseq$ indicate the number of \subjects.

\textit{(a) Model learning:}
We consider a set of sensor sequences 
$\{X_{\lseq}\}_{\lseq=1}^{\nseq}$ 
for $\nseq$ \subjects that 
were measured until the event of interest occurred. 
Each sequence $X_{\lseq}$ 
comprises $\ntime_{\lseq}$ sequential observations,
\small
\begin{align}
\label{eqn:raw_data}
    X_{\lseq}
    =
    \begin{bmatrix}
    | & | & | & & | \\
    x_{\lseq,1} & x_{\lseq,2} & x_{\lseq,3} & \ldots & x_{\lseq, \ntime_\lseq}\\
    | & | & | & & |
	\end{bmatrix},
\end{align}
\normalsize
where $x_{\lseq,\ltime} \in \mathR^{\ldim}$ is 
the $\ltime$-th multivariate observation in the $\lseq$-th \subject 
obtained from $\ldim$-dimensional sensors
\footnote{
Without loss of generality,
the observation $\vecx_{\lseq,\ltime}$ can be set as 
sliding window features with a window size $m$, 
where we can employ 
$[\vecx_{\lseq,\ltime-\window},\ldots, \vecx_{\lseq,\ltime}]$ 
as the observation.}
and $\ntime_{\lseq}$ indicates the event time.
%
We denote 
$X_{\lseq,:\ltime} = \{x_{\lseq,1},\ldots,x_{\lseq,\ltime}\}$ 
as the partial sequence observed up 
until the specific time $\ltime$.
Here, the label $\rul_{\lseq, \ltime}$ 
represents the time interval from the current time $\ltime$ 
to the event time $\ntime_{\lseq}$, i.e., 
$\rul_{\lseq, \ltime} = (\ntime_{\lseq} - \ltime)$.
In other words, 
the label $\rul_{\lseq, \ltime}$ indicates 
the remaining time until the event of interest occurs 
for the $\lseq$-th \subject at time $\ltime$.
Our aim is to learn a model $\Func$ that
can consistently predict the label $\rul_{\lseq, \ltime}$ 
for every time tick $\ltime$ and every \subject $\lseq$
based on the sequential observations $X_{\lseq,:\ltime}$.
%
More specifically, we want to predict 
the event probabilities as a function of time $p_{\lseq,\ltime}(\rul)$ 
to flexibly assess the risk of event occurrence at any given time.
Therefore, letting 
$\alllabeleddata =
\{(X_{\lseq, :\ltime},
\rul_{\lseq,\ltime})
\}_{\lseq,\ltime=1}^{\nseq,\ntime_{\lseq}}$
be a labeled collection,
we formally define our first problem as follows:
\begin{problem}
  \label{problem:leaning}
    \textbf{Given}
    a labeled collection,
    i.e.,
    $\alllabeleddata = 
    \{(X_{\lseq, :\ltime},
    \rul_{\lseq,\ltime})
    \}_{\lseq,\ltime=1}^{\nseq,\ntime_{\lseq}}$,
    \textbf{Learn} a model $\Func$ 
    that maps each $X_{\lseq,:\ltime}$ to 
    the event probabilities as a function of time 
    $p_{\lseq,\ltime}(\rul)$,
    i.e.,
    $p_{\lseq,\ltime}(\rul) = \Func(X_{\lseq,:\ltime})$.
\end{problem}

\textit{(b) Streaming time-to-event prediction:}
Once we have the model $\Func$, 
we aim to achieve a time-to-event prediction 
over multiple data streams 
$\{X_{\newlseq}\}_{\newlseq =1}^{\newnseq}$,
where each $X_{\newlseq}$ is a continuously growing sequence.
For the $\newlseq$-th \subject, the stream observed up to the
current time tick $\ctime$ is denoted as $X_{\newlseq,:\ctime}$, where $\ctime$ increases with
every new time tick.
Formally, our second problem is as follows.

\begin{problem}
  \label{problem:prediction}
    \textbf{Given}
    a data stream $X_{\newlseq,:\ctime}$ and the learned model $\Func$,
    \textbf{Predict} 
    the event probabilities as a function of time $p_{\newlseq,\ctime}(\rul)$ at every new time tick $\ctime$ and \textbf{Update} the model $\Func$ at every new instance $X_{\newlseq}$.
\end{problem}
    \label{040model}
    \subsection{\method}
We now present \methodmodel model, $\Func$,
which is designed to satisfy the following properties
for streaming time-to-event prediction:
\bit
\item \textit{Stochastic time-to-event predictor:}
provides event probabilities for each observation
by capturing the underlying stochastic process.
\item \textit{Interdependency-based descriptor}:
characterizes each observation based on 
statistical interdependencies between sensors.
\item \textit{Sequential multi-model structure:}
captures dynamic changes in sequences and enables adaptive prediction through multiple sequentially connected models.
\eit
Figure~\ref{fig:overview} shows an overview of \methodmodel model for a labeled collection $\alllabeleddata$.
The predictor and the descriptor are the building blocks of each model, which we refer to as a stage model.
Each stage model is associated with a specific stage in $\alllabeleddata$. 
The sequential multi-model structure consists
of multiple stage models that are sequentially connected.
Details are provided in the remaining subsections. 
\TSK{
\begin{figure}[t]
    \hspace{-1em}
    \includegraphics[width=1.05\columnwidth]{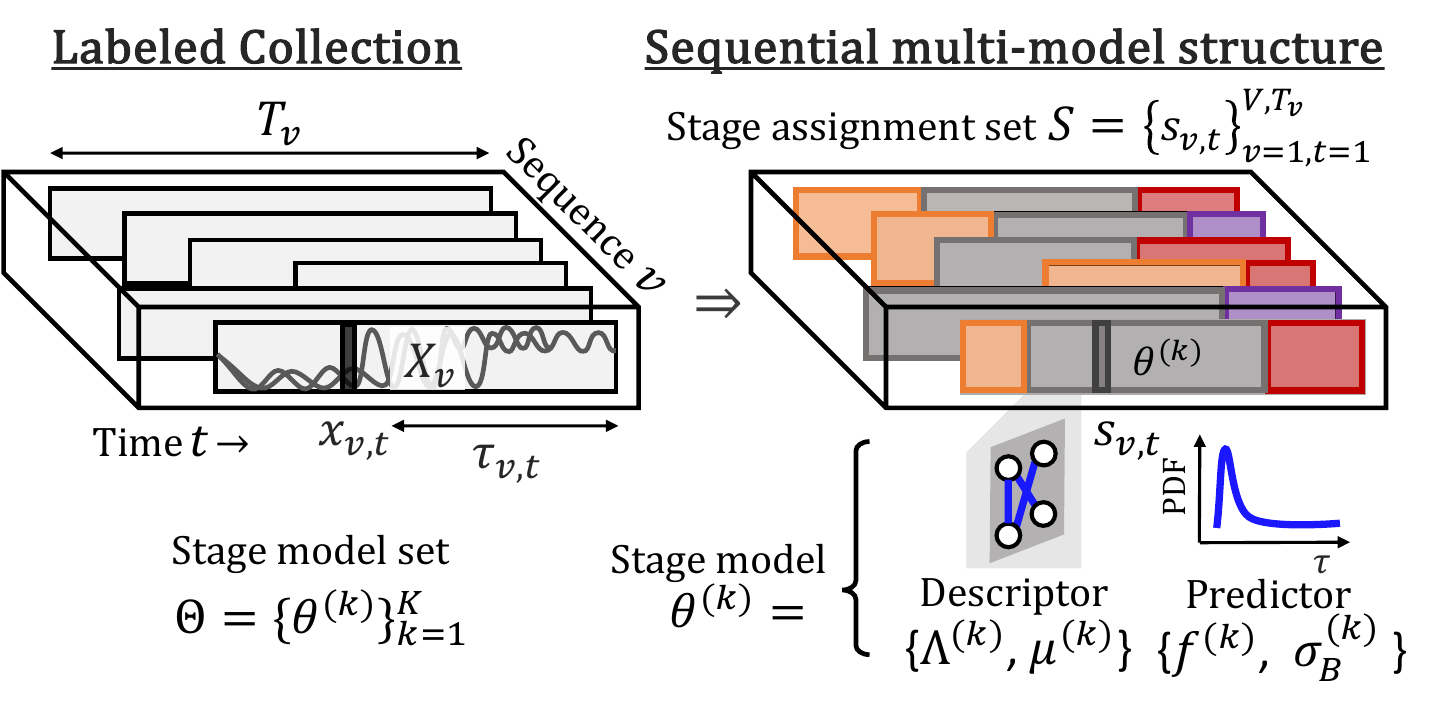}
    \vspace{-2em}
    \caption{
        An overview of \methodmodel for 
        a labeled collection $\alllabeleddata$.
        The method is based on a sequential multi-model structure, 
        which consists of a stage model set $\allmodel=\{\model^{(\lstage)}\}_{\lstage=1}^{\nstage}$ 
        and a stage assignment set $\allstageas$.
        It adopts a different stage model depending on time-varying behaviors.
        Each stage model $\model^{(\lstage)}$ consists of 
        a descriptor
        $\{\dprecision^{(\lstage)},\dmean^{(\lstage)}\}$ and 
        a predictor
        $\{\pfunc^{(\lstage)},\bstd^{(\lstage)}\}$.
    }
    \label{fig:overview}
\end{figure}
}
\subsubsection{Stochastic Time-to-Event Predictor}
We begin with the simplest case, 
in which we model only a single stage.
The first problem is to predict 
the event probabilities as a function of time $p_{\lseq,\ltime}(\rul)$
for a given observation $x_{\lseq,\ltime}$.
We employ the concept of first hitting time, 
where the event time is defined as the first time 
the underlying progression process reaches a prescribed boundary.
Specifically, 
we define the progression process 
$\plevel(\wtime)$ as a Wiener process,
allowing us to represent the first hitting time 
as an inverse Gaussian distribution \cite{cox1965theory,10.5555/73944}.
The progression process $\plevel(\wtime)$ is written as the following stochastic differential equation:
\small
\begin{align}
\label{eqn:sde_wiener}
d\plevel(\wtime) &= \drift d\wtime 
+ \bstd dB(\wtime),
\end{align}
\normalsize
where 
$\drift$ is the drift parameter capturing the rate of progression,
$\bstd$ is the diffusion parameter representing the uncertainty 
of the progression, and $\{B(\wtime)| \wtime \geq 0\}$ is a standard Brownian motion.
That is, for each $\wtime \geq 0$, $\bstd dB(\wtime) \sim \mathcal{N}(0,\bstd^2d\wtime)$, 
which indicates that uncertainty increases as $\wtime$ increases.

Since a true progression process is unobservable, 
we estimate the progression process $\plevel(\wtime)$ 
from observation $x_{\lseq,\ltime}$.
Specifically, 
we introduce a link function\footnote{
In our experiments, 
we used an orthogonal projection for $\pfunc$,
which has a closed-form solution with the weights $\weight$,
i.e., $\rul = \weight x $.
However, the link function is not constrained, and exploring the potential improvements remains an open problem for future work.
}
$\pfunc$ which maps each $\vecx_{\lseq,\ltime}$
to the event time $\rul_{\lseq,\ltime}$ 
and the progression process $\plevel(\wtime)$ is written as follows:
\small
\begin{align}
\label{eqn:sde_wiener_re}
\plevel(\wtime) = \drift \wtime + \bstd B(\wtime),
~~~~~
\drift = 
\frac{1}{f(x_{\lseq,\ltime})}.
\end{align}
\normalsize

Finally, we want to estimate the probabilities as a function of time $p_{\lseq,\ltime}(\rul)$ that 
the progression process $\plevel(\rul)$ reaches a boundary $c$.
Since Equation~$(\ref{eqn:sde_wiener_re})$ ensures that 
the drift parameter satisfies $0<\drift<1$ and that the boundary $c=1$,
the event probabilities $p_{\lseq,\ltime}(\rul)$ is written as follows: 
\small
\begin{align}
\label{eqn:inverse_gauss}
p_{\lseq,\ltime}(\rul;\pfunc,\bstd,\vecx_{\lseq,\ltime}) 
&=\frac{1}{\sqrt{2\pi\bstd^2\wtime^3}}
\exp\Bigl[{-\frac{\bigr(1-\frac{\rul}{\pfunc(\vecx_{\lseq,\ltime})}\bigr)^2}{2\bstd^2\wtime}}\Bigr].
\end{align}
\normalsize
Here, we refer to $\pfunc$ and $\bstd$ as a \textit{predictor},
which enables us to predict the event probabilities $p_{\lseq,\ltime}(\rul)$
for a given observation $\vecx_{\lseq,\ltime}$.

\subsubsection{Interdependency-based descriptor}
The next question is how to characterize 
each observation $\vecx_{\lseq,\ltime}$ to 
identify changes over sensor sequences.
Real sensor sequences might contain various types of noise
that can distort observed values.
To maintain robustness against noise,
we focus on
interdependencies between sensors 
rather than individual statistics.
Specifically,
we employ Gaussian graphical models (GGMs),
which model
the conditional independence 
between $ \ldim$ sensor variables 
in an observation $x_{\lseq,\ltime} \in \mathR^{\ldim}$.
The model captures the underlying distribution
of each observation $\vecx_{\lseq,\ltime}$,
i.e., $x_{\lseq,\ltime} \sim \mathcal{N}(\dmean,\dprecision^{-1})$,
where 
$\dmean \in \mathR^{\ldim}$
and 
$\dprecision \in \mathR^{\ldim \times \ldim}$ 
indicate the mean and 
the \textit{sparse} precision matrix, respectively.
Each value of the sparse precision matrix 
$\dprecision_{i,j}$ 
can indicate pairwise conditional independence,
that is,
\small
\begin{align}
\label{eqn:ggm2}
\dprecision_{i,j} &= 0 
\Leftrightarrow 
x_{\lseq,\ltime,i} \indep 
x_{\lseq,\ltime,j}~|~x_{\lseq,\ltime,\backslash \{i,j\}},
\end{align}
\normalsize
where $x_{\lseq,\ltime,i}$ denotes the $i$-th 
sensor variable in observation $x_{\lseq,\ltime}$,
resulting in the sparse precision matrix $\dprecision$
being interpreted as the adjacency matrix of a graph that 
describes the interdependencies.
Here, we refer to $\dmean$ and $\dprecision$ as a \textit{descriptor},
which characterizes a given observation $\vecx_{\lseq,\ltime}$.
\subsubsection{Sequential Multi-Model Structure}
Thus far, we have discussed \textit{predictor} and \textit{descriptor},
which provides event probabilities as a function of time 
and interdependency-based representations 
for each observation $\vecx_{\lseq,\ltime}$.
However, the model focuses only on individual observations 
and remains insufficient for capturing whole sensor sequences,
containing various types of stages.
We thus propose a sequence-level model architecture.
%
\begin{definition}[Stage model set: $\allmodel$]
Let $\allmodel$ be a set of $\nstage$ stage models.
The $\lstage$-th stage model  $\model^{(\lstage)}$ consists of a predictor and a descriptor, 
i.e., $\model^{(\lstage)} = 
\{\pfunc^{(\lstage)}, 
\bstd^{(\lstage)}, 
\dmean^{(\lstage)},
\dprecision^{(\lstage)} \}$.
%
\end{definition}
Our model employs 
a different stage model $\model^{(\lstage)}\in \allmodel$ 
that depends on a time-varying stage.
Thus, we also want to determine the assignments of stage models
for each observation $\vecx_{\ltime,\lseq}$.
\begin{definition}[Stage assignment set: $\allstageas$]
Let $\allstageas$ be a set of stage assignments 
$\stageas_{\lseq, \ltime}$.
The stage assignment $\stageas_{\lseq, \ltime}$ represents 
a stage index for each observation $x_{\lseq,\ltime}$,
i.e., $\stageas_{\lseq,\ltime} \in \{1,\ldots, \nstage\}$,
and is constrained by the sequential connectivity as follows:
\small
\begin{align}
\label{eqn:progression_stage}
    \ltime < \ltime' \Rightarrow	
    \stageas_{\lseq, \ltime} \leq \stageas_{\lseq, \ltime'} 
    ~|~\forall \lseq,\ltime,\ltime'.
\end{align}
\end{definition}
\normalsize
%
%
Note that the sequential connectivity of 
Eq~(\ref{eqn:progression_stage})
enforces that the stage assignments 
$\stageas_{\lseq,\ltime}$
never decrease over time, 
providing two key benefits.
\textit{First,} 
it ensures that stage assignments are interpreted as an irreversible progression, 
analogous to real-world processes 
leading to future events
such as disease progression \cite{yoon2016discovery} 
and machine degradation \cite{juodelyte2022predicting}.
\textit{Second,} it maintains temporal consistency by ignoring abrupt fluctuations and repeated changes in stage assignments.

Consequently, 
the complete parameter set of \methodmodel that we want to estimate is as follows.
\begin{definition}[Full Parameter Set of \methodmodel: $\Func$]
Let $\Func$ be a complete set of \methodmodel,
i.e., $\Func = \{\allmodel,\allstageas \}$,
where
$\allmodel$ indicates a stage model set and 
$\allstageas$ indicates a stage assignment set.
\end{definition}

\section{Optimization Algorithms}
    \label{050algo}
    Thus far, we have introduced 
our mathematical concept of \methodmodel.
Next, we
tackle Problem~\ref{problem:leaning} and Problem~\ref{problem:prediction}
by proposing the following 
two algorithms.
\bit
\item 
Learning algorithm for Problem~\ref{problem:leaning}:
Efficiently find the optimal parameter of $\Func$
for a given labeled collection $\alllabeleddata$.
\item
Streaming algorithm for Problem~\ref{problem:prediction}:
Adaptively predict the event probabilities $p_{\newlseq,\ctime}(\rul)$
for a data stream $\matx_{\newlseq,:\ctime}$.
\eit
We first introduce our objective function and then describe the proposed algorithms in detail.

\myparaitemize{Objective Function}
\label{sec:optalgo}
Given a labeled collection $\alllabeleddata$,
we aim to estimate the full parameter set 
$\Func=\{\allmodel,\allstageas\}$ that maximizes the following objective:
\small
\begin{center}
\fbox{
\begin{minipage}{0.95\columnwidth}
\begin{align}
\label{eqn:overall_objective}
 \argmax_{\allmodel, 
\allstageas \nearrow_{\ltime}
 }
 \sum_{\lstage=1}^{\nstage}
\underbrace{
\Obj_d(\alllabeleddata~|~\model^{(\lstage)},\allstageas)
}_{\textrm{Descriptor}}
+\predpara\underbrace{
\Obj_p(
\alllabeleddata~|~\model^{(\lstage)},\allstageas)
 }_{\textrm{Predictor}},
\end{align}
\end{minipage}
}
\end{center}
\normalsize
where $\allstageas\nearrow_{\ltime}$ is a constraint of sequential connectivity in Eq.~(\ref{eqn:progression_stage}) and 
$\predpara \geq 0$ is a coefficient chosen to balance 
the descriptors and the predictors.
The first term is the objective function of the 
descriptor for each stage $\lstage$:
\small
\begin{align}
\label{eqn:overall_objective_d1}
&\Obj_d(\alllabeleddata~|~\model^{(\lstage)},\allstageas)
= \sum_{
\stageas_{\lseq,\ltime} =\lstage
}
\Bigr [
\obj_{d}(x_{\lseq,\ltime}~|~\dmean^{(\lstage)},\dprecision^{(\lstage)})
\Bigr ]
-\lassopara ||\dprecision^{(\lstage)}||_{od,1},
\nonumber
\\
&\obj_{d}(x_{\lseq,\ltime}~|~\dmean^{(\lstage)},\dprecision^{(\lstage)})
= -\frac{1}{2}(x_{\lseq,\ltime} -\dmean^{(\lstage)})^\mathrm{T}
\dprecision^{(\lstage)}(x_{\lseq,\ltime}-\dmean^{(\lstage)}) 
\nonumber
\\
&~~~~~~~~~~~~~~~~~~~~~~~~~~~~~~~~~~
+\frac{1}{2}\log\det\dprecision^{(\lstage)} 
- \frac{\ldim}{2}\log (2\pi),
\end{align}
\normalsize
where 
$\obj_{d}$
is the Gaussian log likelihood 
that $x_{\lseq,\ltime}$ comes from stage $\lstage$, 
and $||\cdot||_{od,1}$ is 
the off-diagonal $\ell_1$-norm, 
which enforces element-wise sparsity for the precision matrix,
regulated by the trade-off parameter $\lassopara \geq 0$.
%
The second term is the log likelihood of 
the predictor for each stage $\lstage$
(up to a constant and scale):
\small
\begin{align}
\label{eqn:overall_objective_p}
&\Obj_p(
\alllabeleddata|\model^{(\lstage)},\allstageas) 
=
\sum_{
\stageas_{\lseq,\ltime} = \lstage
}
\obj_{p}(x_{\lseq,\ltime},\rul_{\lseq,\ltime}~|~\pfunc^{(\lstage)},\bstd^{(\lstage)}),
\nonumber
\\
&\obj_{p}(x_{\lseq,\ltime},\rul_{\lseq,\ltime}~|~\pfunc^{(\lstage)},\bstd^{(\lstage)})
=-
||\rul_{\lseq,\ltime} - \pfunc^{(\lstage)}(x_{\lseq,\ltime})||_2
-\log(\bstd^2) 
\nonumber
\\
&~~~~~~~~~~~~~~~~~~~~~~~~~~~~~~~~~~~~~~~~~
-\frac{1}{\bstd^2}||\frac{1}{\rul_{\lseq,\ltime}}-\mu_\rul^{(\lstage)}||_2,
\end{align}
\normalsize
where $||\cdot||_2$ denotes $\ell_2$-norm 
and $\mu^{(\lstage)}_{\rul}$ denotes
the mean of the increments in the $\lstage$-th progression process.

Notably, the objective function simultaneously evaluates both descriptive quality and prediction accuracy, so the resulting parameters 
$\{\allmodel, \allstageas \}$
reflect both aspects of the data.
This design follows 
the concept of multi-task learning \cite{caruana1997multitask},
where the joint learning of multiple tasks 
serves as an inductive bias that improves generalization.
As discussed in Section~\ref{060experiments},
the joint learning with both parts improved prediction performance in our experiments.
\subsection{Learning Algorithm}
Our first goal is to estimate the full parameter set $\Func$
to maximize the objective function in \eq{\ref{eqn:overall_objective}}.
However, this problem is combinatorial and non-differentiable, rendering widely-used SGD-based methods inapplicable.
Instead, we propose an efficient learning algorithm that exhibits stable and monotonic optimization behavior.
Algorithm~\ref{alg:main} shows the overall procedure,
where
we first initialize $\{\allmodel,\allstageas\}$ to some random values
and then iteratively update subsets of parameters.
First, we update the stage model set $\allmodel$ 
while keeping the stage assignments $\allstageas$ fixed
(lines 3-5).
Second, we update $\allstageas$ with the fixed parameters of
$\allmodel$
(lines 6-8).
We iterate these two steps until convergence.
We describe each step in detail in the following subsections. 

\begin{algorithm}[t]
    \footnotesize
    \caption{Learning Algorithm 
    $(\alllabeleddata, \nstage, \lassopara,\predpara)$}
    \label{alg:main}
\begin{algorithmic}[1]
    \REQUIRE
        (a) Labeled collection
        $\alllabeleddata = 
        \{(X_{\lseq, :\ltime},
        \rul_{\lseq,\ltime})\}_{\lseq,\ltime=1}^{\nseq,\ntime_{\lseq}}$
        \\ \hspace{1.35em}
        (b) Initial number of stages $\nstage$ 
        \\ \hspace{1.35em}
        (c) Sparse parameter $\lassopara$
        \\ \hspace{1.35em}
        (d) Balance parameter $\predpara$
    \ENSURE
        Full parameter set $\Func$
    \STATE $ \{\allmodel, \allstageas\} \leftarrow 
    \textsc{Initialize}(\alllabeleddata, \nstage, \lassopara,\predpara)$;
    \REPEAT
    \FOR{$\lstage$-th stage}
        \STATE $\allmodel \leftarrow \textsc{UpdateStageModels}
        (\alllabeleddata,\allstageas,\lassopara)$;
        {\color{blue}
        // Section~\ref{sec:upd_model}
        }
    \ENDFOR
    \FOR{$\lseq$-th sequence}
        \STATE $\allstageas \leftarrow 
        \textsc{UpdateAssignments}(X_{\lseq},\allmodel,\predpara)$;
        {\color{blue}
        // Section~\ref{sec:upd_assign}
        }
    \ENDFOR
    \UNTIL{convergence;}
    \RETURN $\Func$;
\end{algorithmic}
\normalsize
\end{algorithm}

\subsubsection{\textsc{UpdateStageModels}}
\label{sec:upd_model}
In this step, we estimate the parameters for all stage models
$\{\model^{(\lstage)}\}_{\lstage=1}^{\nstage}$,
while fixing the stage assignments $\allstageas$.
Once the stage assignments are fixed, 
we can optimize each stage $\model^{(\lstage)}$ independently 
by maximizing \eq{\ref{eqn:overall_objective_d1}} and \eq{\ref{eqn:overall_objective_p}}.
Specifically,
maximizing \eq{\ref{eqn:overall_objective_d1}} 
is equivalent to the graphical lasso problem \cite{glasso2008}.
Since this is a convex optimization problem,
we use the alternating direction method of multipliers \cite{boyd2011distributed},
which efficiently converges on the globally optimal solution.
%
%
%
In addition, we maximize \eq{\ref{eqn:overall_objective_p}} 
through maximum likelihood estimation.
The details of each maximization problem 
are described in Appendix~\ref{sec:app:algo}.

\subsubsection{\textsc{UpdateAssignments}}
In this step, we find the optimal stage assignments $\allstageas$,
while fixing the value of $\allmodel$.
We rewrite our maximization problem
(i.e., \eq{\ref{eqn:overall_objective}})
in terms of the stage assignments $\allstageas$,
which is written as follows for each sequence $\matx_{\lseq}$: \label{sec:upd_assign}
\TSK{
\begin{figure}[t]
    \hspace{-1.5em}
    \centering
    \includegraphics[width=0.8\columnwidth]{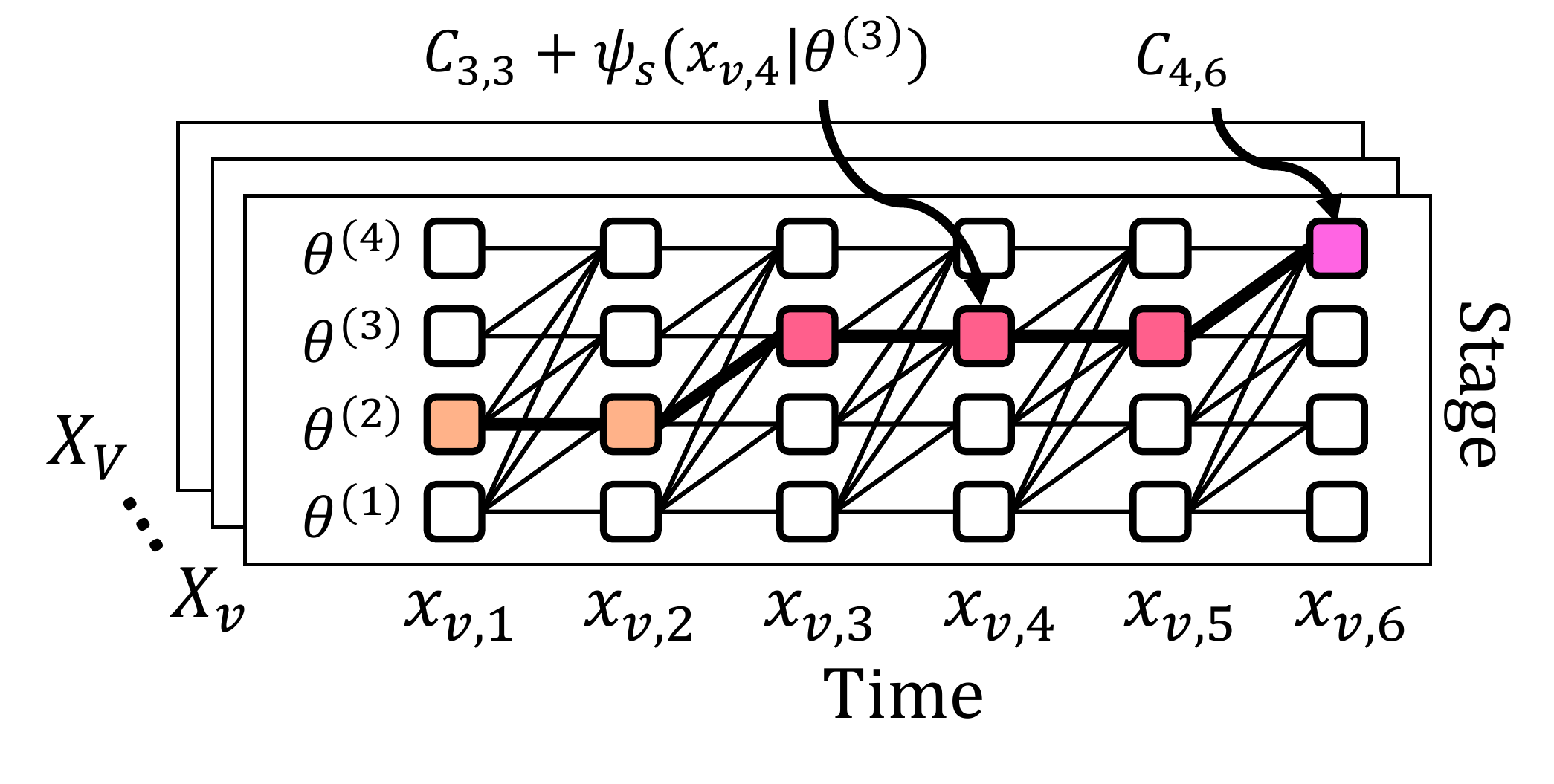}
    \vspace{-1.5em}
    \caption{
    Dynamic programming algorithm for stage assignments.
    The algorithm efficiently finds the optimal stage assignments
    by sequentially computing the cost $\optcost{\lstage}{\ltime}$.
    }
    \label{fig:path}
\end{figure}
} 
\small
\begin{align}
\label{eqn:stage_assign}
\argmax_{\stageas_{\lseq,\ltime} \nearrow_{\ltime}}
&\sum_{\ltime=1}^{\ntime_{\lseq}}~
\cost{x_{\lseq,\ltime}}{\model^{(\stageas_{\lseq,\ltime})}},
\\
\nonumber
\cost{x_{\lseq,\ltime}}{\model^{(\stageas_{\lseq,\ltime})}}=~
&\obj_{d}(x_{\lseq,\ltime}~|~\dmean^{(\stageas_{\lseq,\ltime})},
\dprecision^{(\stageas_{\lseq,\ltime})})
\\
\nonumber
&+\predpara \obj_{p}
(x_{\lseq,\ltime},\rul_{\lseq,\ltime}~|~
\pfunc^{(\stageas_{\lseq,\ltime})},\bstd^{(\stageas_{\lseq,\ltime})}).
\end{align}
\normalsize
Since each stage assignment $\stageas_{\lseq,\ltime}$ is 
constrained by the sequential connectivity in 
Eq.~(\ref{eqn:progression_stage}),
Eq.~(\ref{eqn:stage_assign}) is a combinatorial optimization problem
that requires finding
the optimal assignments of $\nstage$ stage models 
to $\ntime_{\lseq}$ observations.
However,
the number of possible assignments is 
$O(\nstage^{\ntime_{\lseq}})$, making it computationally prohibitive.
Therefore, we introduce 
a dynamic programming algorithm
that finds a globally optimal solution 
in only $O(\nstage^2\ntime_{\lseq})$ operations.
Specifically, 
we sequentially compute a cost $\optcost{\lstage}{\ltime}$,
which is provided as follows:
\small
\begin{align}
\label{eqn:assign_cost}
\hspace{0.75em}
\optcost{\lstage}{\ltime} = 
\begin{cases}
\cost{x_{\lseq,1}}{\model^{(\lstage)}} 
\hspace{10.6em}
(\ltime=1)
\\
\max_{1\leq \llstage \leq\lstage}\{\optcost{\llstage}{\ltime-1}\}
~+ \cost{x_{\lseq,\ltime}}{\model^{(\lstage)}}
\hspace{0.7em}
(2\leq \ltime \leq \ntime_{\lseq})\\
\end{cases}
\end{align}
\normalsize
When computing the cost $\optcost{\lstage}{\ltime}$, 
we also record the path of the stage assignments.
After computing $\optcost{\lstage}{\ntime_\lseq}$ for all stages $\nstage$,
we can find the optimal stage assignments by choosing the path
that gives the maximum cost, 
i.e., $\max_{1 \leq \lstage \leq \nstage}\{\optcost{\lstage}{\ntime_\lseq}\}$.
This procedure is illustrated as a lattice diagram in Figure~$\ref{fig:path}$, where the stages are on the vertical axis and the time on the horizontal axis.
At each time tick $\ltime$, 
$\optcost{\lstage}{\ltime}$ 
provides the path that maximizes the cost to reach the $\lstage$-th stage 
among possible assignments.

Overall, our learning algorithm iteratively updates the stage models and assignments. Each update improves the objective in a monotonic manner, leading to a stable optimization process.
%
\begin{lemma}[Proof in \appen{\ref{sec:app:complexity_learning}}]
\label{complexity_learning}
The time complexity of the learning algorithm 
in \method is $O(\#iter \cdot \sum_{\lseq}\ntime_{\lseq})$.
\end{lemma}
\subsection{Streaming Algorithm}
We now address Problem~\ref{problem:prediction},
namely streaming time-to-event prediction.
%
Assuming an observation $\vecx_{\newlseq, \ctime}$ 
is continuously obtained as the current observation of a data stream
for the $\newlseq$-th instance
$\matx_{\newlseq, :\ctime}$,
our aim is to predict event probabilities $p_{\newlseq,\ctime}(\rul)$  
while updating the stage model set $\allmodel$ 
to maintain prediction performance.
%
Algorithm~\ref{alg:scan} outlines the overall procedure, which consists of two steps:

(1) \textsc{AdaptivePredict}:
Continuously estimates the current stage 
$\stageas_{\newlseq,\ctime}$
from the current observation
$x_{\newlseq,\ctime}$ 
in the context of the data stream
$X_{\newlseq,:\ctime}$.
The current stage
$\stageas_{\newlseq,\ctime}$
is identified by maximizing the following equation:
\small
\begin{align}
\label{eqn:stage_assign_stream}
\stageas_{\newlseq,\ctime} 
\leftarrow 
\argmax_{\{\stageas_{\newlseq,\ltime}\}\nearrow_{\ltime}}
\sum_{\ltime=1}^{\ctime}
~~\obj_{d}(\vecx_{\newlseq,\ltime}~|~\dmean^{(\stageas_{\newlseq,\ltime})},
\dprecision^{(\stageas_{\newlseq,\ltime})}).
\end{align}
\normalsize
Although estimating the current stage
$\stageas_{\newlseq,\ctime}$ 
requires past observations $\{\vecx_{\newlseq,1},\ldots,\vecx_{\newlseq,\ctime-1}\}$,
accessing them at each prediction step is computationally expensive in streaming environments.
Therefore, we solve \eq{\ref{eqn:stage_assign_stream}} in an online manner.
Similar to \eq{\ref{eqn:assign_cost}}, we use a dynamic programming algorithm,
which is written as follows:
\small
\begin{align}
\label{eqn:stage_assign_stream_help}
\optcost{\lstage}{\ctime} = 
&
\begin{cases}
\hspace{0.5em}
\obj_{d}(x_{\newlseq,\ctime}~|~\dmean^{(\lstage)},
\dprecision^{(\lstage)}),
\hspace{10.6em}
(\ctime=1)
\\
\hspace{0.5em}
\max_{1\leq \llstage \leq\lstage}\{\optcost{\llstage}{\ctime-1}\}
~+ \obj_{d}(x_{\newlseq,\ctime}~|~\dmean^{(\lstage)},
\dprecision^{(\lstage)})
\hspace{1em}
(2\leq \ctime)
\nonumber\\
\end{cases}
\\
&\hspace{5em}
\stageas_{\newlseq,\ctime} 
\leftarrow 
\argmax_{1\leq \lstage \leq\nstage}~~
\optcost{\lstage}{\ctime}.
\end{align}
\normalsize
This procedure enables us to estimate the current stage $\stageas_{\newlseq,\ctime}$
based on the current observation $\vecx_{\newlseq,\ctime}$
and the cost set $\{\optcost{\lstage}{\ctime-1}\}_{\lstage=1}^{\nstage}$
at previous time $\ctime-1$.
For the next time tick $\ctime + 1$, 
we retain the cost set $\{\optcost{\lstage}{\ctime}\}_{\lstage=1}^{\nstage}$
and discard the cost set $\{\optcost{\lstage}{\ctime-1}\}_{\lstage=1}^{\nstage}$.
Finally, 
the algorithm predicts the event probabilities as a function of time
$p_{\newlseq,\ctime}(\rul)$ 
by exploiting the stage model 
$\model^{(\stageas_{\newlseq,\ctime})}$.

(2) \textsc{OnlineModelUpdate:} Runs when the data stream $X_{\newlseq, :\ctime}$ is observed up to the event time $\ntime_{\newlseq}$, i.e., when 
$\{(X_{\newlseq, :\ltime}, \rul_{\newlseq,\ltime})
\}_{\ltime=1}^{\ntime_{\newlseq}}$ are available.
To adapt non-stationary data streams, 
this step employs a generate-and-validate approach that 
generates a new stage model set $\allmodel^{+}$ 
and adopts the models only if this leads to improved prediction accuracy.
Specifically, we first initialize the new stage model $\model^{+}$ based on observations 
assigned to the stage with the worst prediction accuracy.
An augmented model set 
$\allmodel^{+}=\{\allmodel,\model^{+}\}$
is updated by iterating \textsc{UpdateAssignments} and \textsc{UpdateStageModels}.
Given the existing model set $\allmodel$ and
the augmented model set 
$\allmodel^{+}$,
we compare their prediction accuracies on $X_{\newlseq}$.
Note that estimating the existing stage models 
$\{\model^{(\lstage)}\}_{\lstage=1}^{\nstage}$
requires observations assigned to each stage in the learning algorithm. 
Owing to the careful design of the stage models based on the means and covariances of observations, 
the parameters of 
$\{\model^{(\lstage)}\}_{\lstage=1}^{\nstage}$ 
can be updated online using Welford's algorithm \cite{welford1962note}.
\begin{lemma}[Proof in \appen{\ref{sec:app:complexity_streaming}}]
\label{complexity_streaming}
The time complexity of the streaming algorithm
in \method is 
$O((1 + \#iter) \cdot \nstage^2)$
amortized per time step.
\end{lemma}

\TSK{
\begin{algorithm}[t]
    \footnotesize
    \caption{Streaming Algorithm
    $(x_{\newlseq,\ctime},
    \allmodel,
    \{\optcost{\lstage}{\ctime-1}\}^{\nstage}_{\lstage=1})$}
    \label{alg:scan}
\begin{algorithmic}[1]
    \REQUIRE
        (a) Recent observation $x_{\newlseq, \ctime}$
        \\ \hspace{1.4em}
        (b) Stage model set $\allmodel$
        \\ \hspace{1.4em}
        (c) Previous cost set $\{\optcost{\lstage}{\ctime-1}\}^{\nstage}_{\lstage=1}$
    \ENSURE
        (a) Predicted event probabilities
        $p_{\newlseq,\ctime}(\rul)$ 
        \\ \hspace{2.05em}
        (b) Updated stage model set $\allmodel$
        \\ \hspace{2.05em}
        (c) Updated cost set $\{\optcost{\lstage}{\ctime}\}^{\nstage}_{\lstage=1}$
    \STATE$p_{\newlseq,\ctime}(\rul),
    \{\optcost{\lstage}{\ctime}\}^{\nstage}_{\lstage=1}
    \leftarrow\textsc{AdaptivePredict}(x_{\newlseq,\ctime},\allmodel,\{\optcost{\lstage}{\ctime-1}\}^{\nstage}_{\lstage=1})$;
    \STATE 
    \textbf{if} {$\ctime == \ntime_{\newlseq}$} \textbf{then}
    {\color{blue}
    ~// \textsc{OnlineModelUpdate}
    }
    \STATE \quad
    $\model^{+} \leftarrow 
    \textsc{Initialize}(X_{\newlseq})$;~~
    $\allmodel^{+} \leftarrow \{\allmodel,\model^{+}\}$;
    \STATE \quad 
    \textbf{repeat} 
    \STATE \quad\quad
    $\allstageas
    \leftarrow \textsc{UpdateAssignments}(X_{\newlseq},\allmodel^{+},\predpara)$;
    \STATE \quad\quad
    $\allmodel^{+} \leftarrow \textsc{UpdateStageModels-Online}
    (
X_{\newlseq},
\allstageas,\lassopara,\allmodel^{+})$;
    \STATE \quad 
    \textbf{until} convergence;
    
    \STATE \quad
    \textbf{if}
    {$\allmodel^{+}$ improves prediction accuracy}
    \textbf{then}
    \STATE \quad\quad $\allmodel \leftarrow \allmodel^{+}$
    \STATE \quad \textbf{end if}
    \STATE \textbf{end if}
    \RETURN $p_{\newlseq,\ctime}(\rul)$,$\allmodel$,$\{\optcost{\lstage}{\ctime}\}^{\nstage}_{\lstage=1}$;
\end{algorithmic}
\normalsize
\end{algorithm}
}
\section{Experiments}
    \label{060experiments}
    \begin{figure*}[t]
    \centering
     \includegraphics[width=1\linewidth]{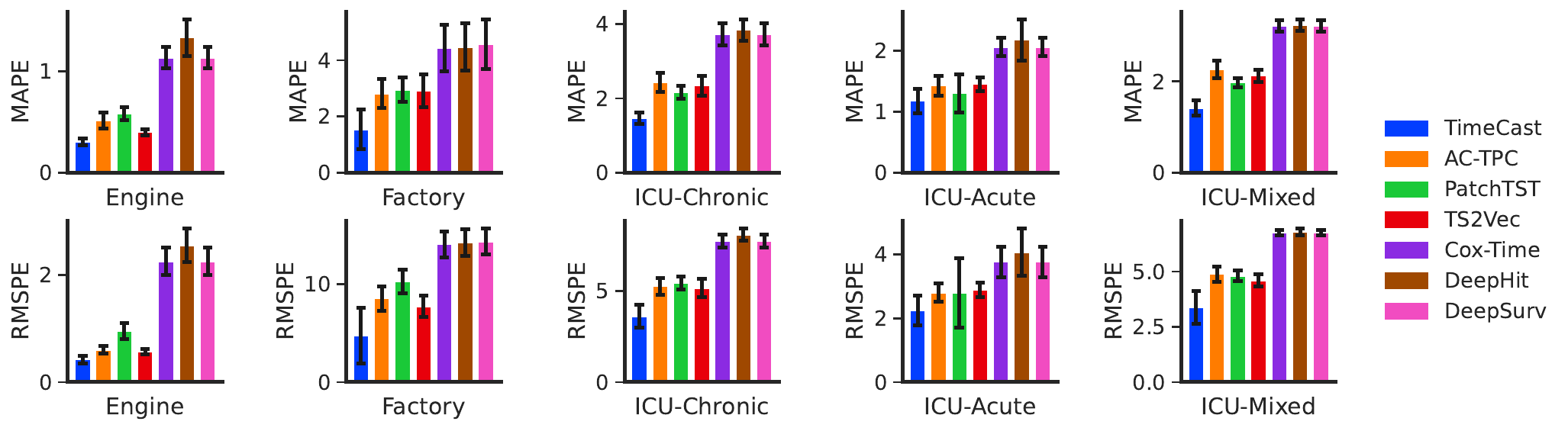}\\
    \vspace{-1.5em}
    \caption{  
    Comparison of prediction performance.
\method consistently outperforms its baselines (lower is better).
    }
    \label{fig:compare_all}
\end{figure*}
In this section, 
we evaluate the performance of \method.
We answer the following questions through 
experiments.
\begin{itemize}
    \item[(Q1)]
        \textit{Accuracy}:
        How accurately does it achieve 
        streaming time-to-event prediction?
    \item[(Q2)]
    \textit{Scalability}:
        How does it converge
        and scale in terms of computational time?
    \item[(Q3)]
    \textit{Real-world Effectiveness}:
        How does it provide meaningful discoveries through stage identification?    
\end{itemize}
%
\myparaitemize{Experimental Settings}
\TSK{
\begin{table}[t]
    \centering
    \caption{Dataset description}
    \label{table:datasets}
    \vspace{-0.5em}
    \vspace{-0.5em}
\scalebox{0.88}{
    \begin{tabular}{l|cccc}
        \toprule        
        Dataset & $\nseq + \newnseq$ & \hspace{2em}$\ldim$ & $\sum_{\lseq} \ntime_{\lseq}$ & $Avg(\ntime_{\lseq})$\\
        \midrule
        \multicolumn{4}{l}{
        Industrial dataset: 
        \textit{(machine, sensor, time)} 
        $\rightarrow$ Failure
        }\\
        \midrule
        \#1 \cmapss
        &  $200$ & \hspace{2em} $7$ & $45,351$ 
        & $227\pm72$ \\
        \#2 \mapm
        &  $98$& \hspace{2em} $4$ & $89,538$ 
        & $914\pm710$ \\
        \midrule
        \multicolumn{4}{l}{
        Medical dataset: 
        \textit{(patient, vital sign, time)} 
        $\rightarrow$ Mortality
        }\\
        \midrule
        \#3 \chronic
        &  $355$ & \hspace{2em} $6$ & $98,320$ 
        & $277\pm217$ \\
        \#4 \acutedata
        &  $112$ & \hspace{2em} $6$  & $20,193$ 
        & $180\pm96$ \\
        \#5 \mixed 
        &  $521$ & \hspace{2em} $6$  & $141,336$ 
        & $271\pm194$ \\
        \bottomrule
    \end{tabular}
}
 \normalsize
\end{table}
}
We use five publicly available real-world datasets listed in Table~\ref{table:datasets},
consisting of sensor sequences 
recorded until a particular event occurs
in mechanical systems and patients at ICUs. 
The six baseline methods are as follows:
DeepSurv~\cite{katzman2018deepsurv},
DeepHit~\cite{lee2018deephit}, and 
Cox-Time~\cite{kvamme2019time},
which are time-to-event prediction methods.
We also compared our method with 
TS2Vec~\cite{yue2022ts2vec}, 
a time series representation learning method; 
PatchTST~\cite{Yuqietal-2023-PatchTST}, 
a transformer-based time-series modeling approach;
and
AC-TPC~\cite{lee2020temporal},
a predictive clustering method.
We use scale-invariant performance metrics, MAPE and RMSPE,
based on percentage errors 
between the predicted event time $\hat{\rul}_{\newlseq,\ltime}$ and 
the true event time $\rul_{\newlseq,\ltime}$.
Although the true event time $\rul_{\newlseq,\ltime}$ 
can be a larger value depending on the sequence length,
these metrics allow for
consistent comparison across different scales.
Lower values indicate better prediction accuracy.
For evaluations, we apply 5-fold cross validation. 
We randomly separated the \subjects into a training set (80\%) and a testing
set (20\%). We reserved 10\% of the training set as a validation set.
The hyperparameters were selected 
based on the prediction performance on the validation set.
We used the parameters of \method for 
$\lassopara=1$, $\predpara=0.1$, and 
$\nstage=5$.
Detailed experimental settings, 
including data preprocessing and baseline parameters,
are provided in
Appendix~\ref{sec:app:exp}. 

\myparaitemize{Q1. Accuracy}
We compared the prediction performance of 
\method 
with that of the baselines. 
Figure~\ref{fig:compare_all} shows the MAPE and RMSPE 
on all the datasets.
For the methods that provide 
the event probabilities $p_{\newlseq,\ltime}(\rul)$,
we employ the mean of $p_{\newlseq,\ltime}(\rul)$
as the predicted event time $\hat{\rul}_{\newlseq,\ltime}$. 
Our method consistently outperforms its baselines
because it can capture non-stationary data streams through 
a sequential multi-model structure.
DeepSurv, DeepHit, and Cox-Time are static time-to-event prediction methods that focus only on individual observations.
They fail to capture dynamic changes and sequential features 
for given data streams.
Although TS2Vec and PatchTST effectively learn the contextual representation of sequences, they cannot distinguish multiple stages.
AC-TPC is a predictive clustering method that makes 
predictions while finding clusters.
However, the method is capable of modeling
sequential connectivity between clusters, 
leading to suboptimal results in streaming 
time-to-event prediction.

\TSK{
\begin{figure}[t]
    \centering
    \hspace{0em}
     \includegraphics[width=1\columnwidth]{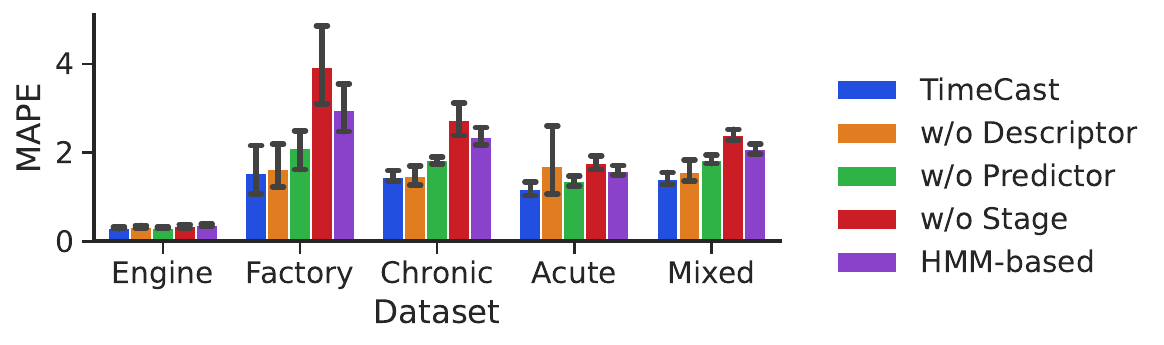}\\
    \vspace{-1.5em}
    \caption{
    Prediction accuracy of \method and its variants 
    on MAPE.
    Each component improves the prediction performance on all datasets 
    (lower is better).
    }
    \label{fig:ablation}
    \begin{tabular}{cc}
    \hspace{-0.5em}
    \begin{minipage}{0.5\columnwidth}
    \centering
     \includegraphics[width=1\linewidth]{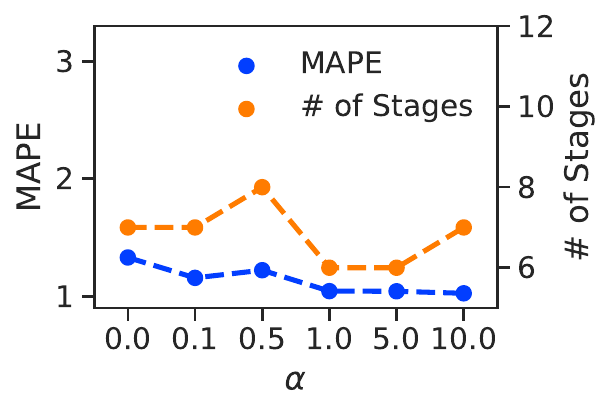}\\
     \end{minipage}
    &
    \hspace{-1.5em}
    \begin{minipage}{0.5\columnwidth}
     \centering
     \includegraphics[width=1\linewidth]{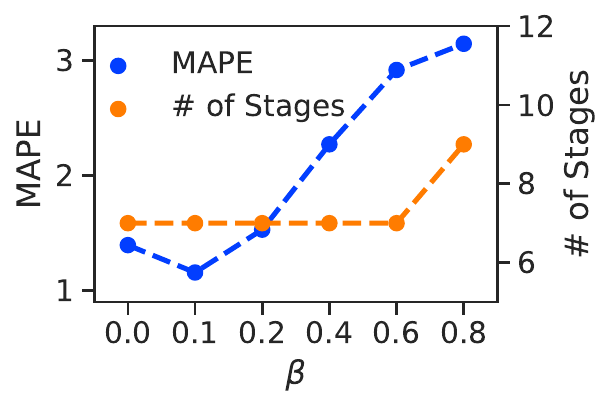}\\
     \end{minipage}
    \end{tabular}
    \vspace{-1.5em}
    \caption{
    Hyperparameter sensitivity of \method.
    }
    \label{fig:parameter}
\end{figure}
}
\myparaitemize{Ablation Study}
To verify the effectiveness of the proposed components in \method,
we conducted ablation studies on all the datasets.
Figure~\ref{fig:ablation} shows 
the prediction accuracy of \method and 
its variants, 
which learn $\Func$ while excluding the effect of a specific component.
Specifically, 
\textbf{(a) w/o Stage} removes the sequential connectivity of \eq{\ref{eqn:progression_stage}},
\textbf{(b) w/o Predictor} removes the effect of the predictor (i.e., $\predpara=0$),
\textbf{(c) w/o Descriptor} learns the stage models without imposing sparsity 
on the precision matrices (i.e., $\lassopara=0$),
and \textbf{(d)  HMM-based} uses HMM to find stages 
instead of the proposed stage models.
The results show that the proposed components are complementary, 
i.e., joint optimization with all parts improves the prediction performance.

\myparaitemize{Hyperparameter Sensitivity}
We analyze the sensitivity of \method
to its hyperparameters.
\fig{\ref{fig:parameter}} shows the prediction results 
when varying hyperparameter settings on the \acutedata dataset.
In the left part of Figure~\ref{fig:parameter}, 
a larger value of the sparsity parameter $\lassopara$ leads to learn more robust descriptors against noise, resulting in a slight improvement in MAPE. 
A detailed analysis of the effect of sparsity is provided in Appendix~\ref{sec:app:results}.
The parameter $\predpara$ indicates the effects of the predictors 
in both the learning algorithms and online model updates.
The right part of Figure~\ref{fig:parameter} shows 
that the larger value of $\predpara$ degrades the MAPE 
while only marginally affecting the number of stages.

\myparaitemize{Q2. Scalability}
\label{sec:scalability}
We evaluate the efficiency of \method.
We first show that 
the learning algorithm 
converges within a small number of iterations.
The left part of \fig{\ref{fig:saclability}} shows the value of our objective function 
(i.e., Eq.~(\ref{eqn:overall_objective})) in each iteration
on the \acutedata dataset.
Thanks to our efficient optimization,
even with $20$ stages, it converged within $20$ iterations.
The right part of \fig{\ref{fig:saclability}} shows 
the computational time for the learning algorithm 
when we vary the total duration of the sequences
on the \mixed dataset.
Since it takes $O(\#iter \cdot \sum_{\lseq}\ntime_{\lseq})$ time 
(as discussed in Lemma~\ref{complexity_learning})
and $\#iter$ is small in practice, 
the complexity scales linearly with respect to the data length.
\fig{\ref{fig:compare_time}} compares 
the prediction time with its competitors as regards computational time 
on all the datasets.
Our method outperforms its baselines 
in speed by up to four orders of magnitude,
enabling rapid response even when sensor readings arrive at high rates.
\TSK{
\begin{figure}[t]
    \centering
    \begin{tabular}{cc}
    \hspace{-2em}
    \begin{minipage}{0.43\columnwidth}
    \centering
     \includegraphics[width=1\linewidth]{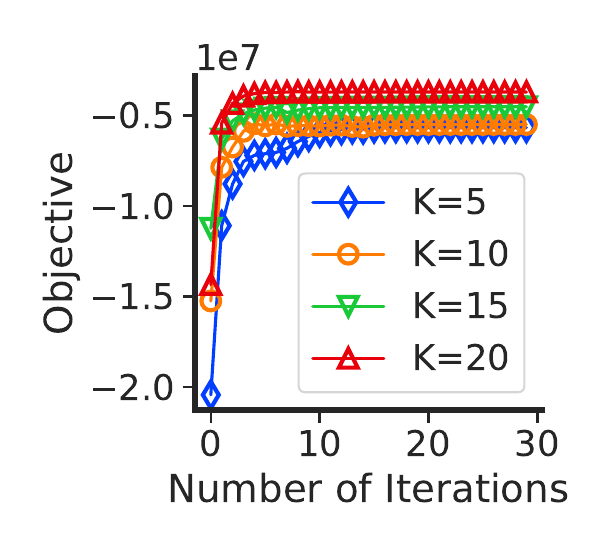}\\
     \end{minipage}
    &
    \hspace{-1.5em}
    \begin{minipage}{0.61\columnwidth}
    \vspace{1em}
     \centering
     \includegraphics[width=1\linewidth]{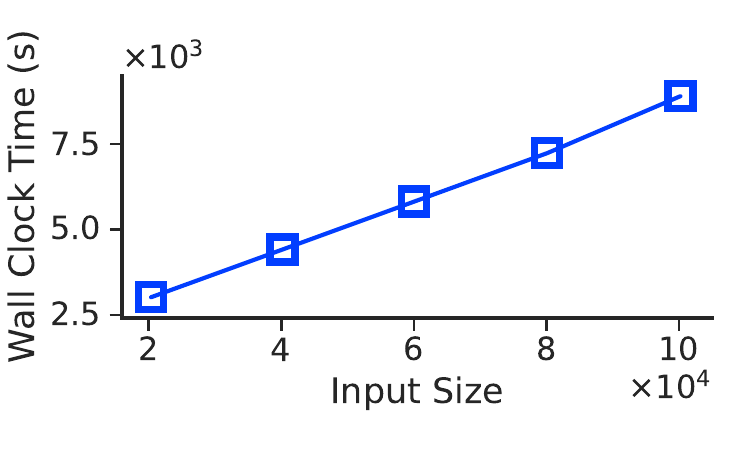}\\
     \end{minipage}
    \end{tabular}
    \vspace{-1.75em}
    \caption{
    Scalability of \method.
    (left) Fast convergence of our learning algorithm.
    It converged within $20$ iterations in the \acutedata dataset.
    (right) Wall clock time vs. input size. 
    The learning algorithm of \method scales linearly.
    }
    \label{fig:saclability}
     \includegraphics[width=1\columnwidth]{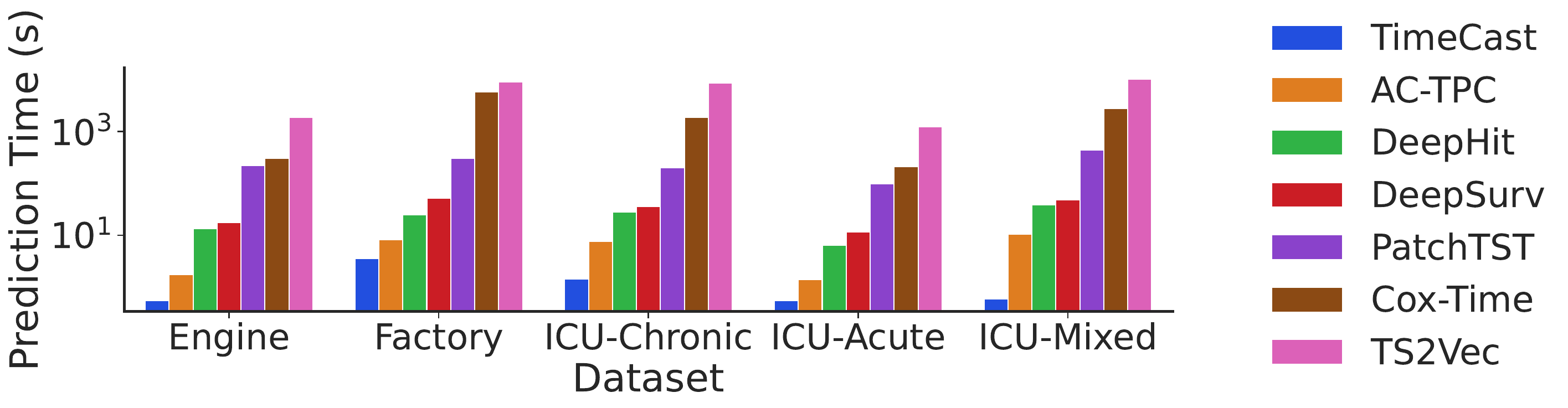}\\
    \vspace{-1em}
    \caption{
    Prediction time for all datasets.
    \method consistently faster than its baselines.
    The results are shown in log scale.
    }
    \label{fig:compare_time}
\end{figure}
}

\myparaitemize{Q3. Real-World Effectiveness}
\begin{figure}[t]
    \centering
     \centering
     \hspace{-2em}
    \includegraphics[width=1\columnwidth]{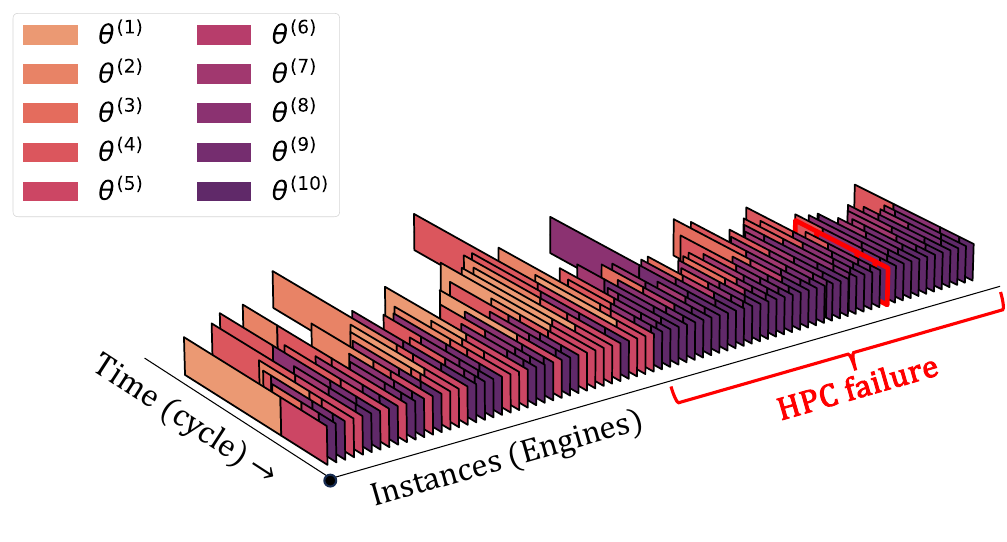}\\
    \vspace{-1.5em}
    \caption{Stage identification of \method.
    The method identifies ten stages, shown as colored segments.
    Here, the failure-specific evolution pattern appears on multiple instances,
    highlighted by the red bracket.
}
    \label{fig:segments}
    \centering
    \hspace{-1em}
     \includegraphics[width=1.05\columnwidth]{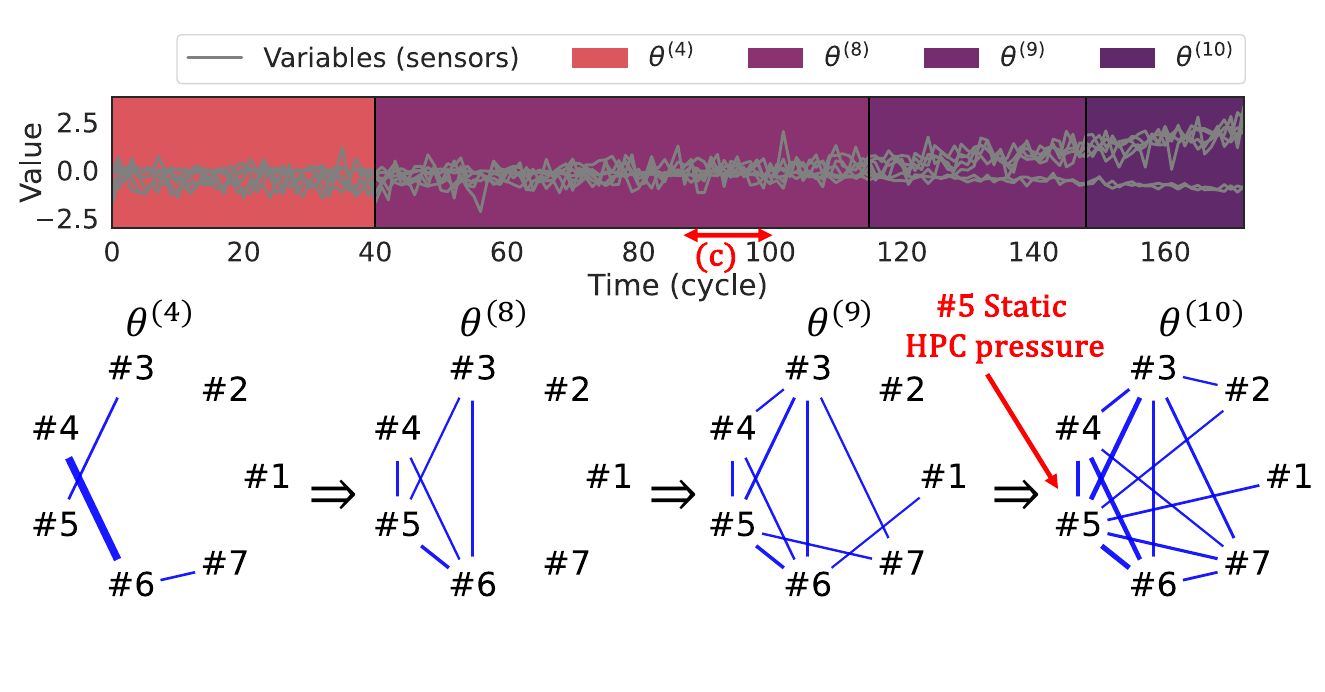}\\
    \vspace{-2.5em}
    \caption{Time-varying interdependencies between seven sensor variables. 
    \method finds that 
    variations in the other sensors depend on the static pressure at the HPC
    just before HPC failure.
    }
    \label{fig:dependency}
\end{figure}
The prediction results for the \cmapss dataset have already been 
presented in Figure~\ref{fig:preview}.
The figure demonstrated that 
\method continuously provides the probabilities for future failure time,
leading to immediate shutdown and maintenance scheduling.
We here provide some of our discoveries on stage identification of \method,
which allow us to understand how the conditions of turbofan engines 
change over time.

Figure~\ref{fig:segments} shows 
stage identification for multiple \subjects (i.e., engines).
\method discovers ten stages (i.e., $\model^{(1)},\ldots,\model^{(10)}$)
and their shifting points, 
where the assignments of each stage  
are indicated as a set of colored segments,
and each engine is aligned along the failure time.
Here, we observed failure-specific behavior:
multiple engines indicated by 
a red bracket have similar evolutionary stages
(i.e., $\model^{(4)}\rightarrow \model^{(8)}\rightarrow \model^{(9)}\rightarrow \model^{(10)}$).
According to the investigation of the dataset \cite{saxena2008damage},
these engines experienced the same type of failure 
called high-pressure compressor (HPC) failure. 

Each stage model $\model^{(\lstage)}$ 
captures interdependencies between sensor variables.
Figure~\ref{fig:dependency} illustrates 
the time-varying interdependencies for an engine that experienced HPC failure,
where 
the dependencies are visualized as a graph.
The nodes indicate individual sensor variables,
and the edge widths indicate the connection intensities.
Notably, the number of edges of \#5 (static pressure at the HPC) 
gradually increases and 
is connected to every other node in the last stage.
This means that just before HPC failure,
variations in the other sensors depend on the static pressure at the HPC.
\section{Conclusion}
    \label{070conclusions}
    In this paper, 
we focused on streaming time-to-event prediction and presented \method,
which exhibits all the desirable properties that we listed in the introduction;
\textbf{(a) Dynamic}: 
\method adaptively 
provides event probabilities at future time points
while detecting and updating stage shifts in data streams.
\textbf{(b) Practical}:
The method continuously predicts the event time with high accuracy while providing semantic information about the data.
\textbf{(c) Scalable}:
The learning algorithm showed fast convergence and 
linear scalability with respect to data size.
The streaming algorithm makes predictions while efficiently updating the model structure.
Our experimental evaluation using five real datasets showed that \method outperforms existing methods in terms of accuracy and execution speed.
The framework of \method is general and flexible, opening up new possibilities
for time-to-event prediction over data streams.
Exploring alternative models for predictors and descriptors,
as well as extending the framework to other application domains,
constitutes an important direction for future research.

\begin{acks}
The authors would like to thank the anonymous referees
for their valuable comments and helpful suggestions.
This work was partly supported by
JSPS KAKENHI Grant-in-Aid for Scientific Research Number JP20910053, JP25K21208
JST CREST JPMJCR23M3,
JST START JPMJST2553,
JST CREST JPMJCR20C6,
JST K Program JPMJKP25Y6,
JST COI-NEXT JPMJPF2009,
JST COI-NEXT JPMJPF2115,
the Future Social Value Co-Creation Project - Osaka University.
\end{acks}

\bibliographystyle{ACM-Reference-Format} 

\appendix
\section*{Appendix}
\section{Algorithms}
\label{sec:app:algo}
\subsubsection{Updating Stage Models}
We describe how to update stage models in detail.
In this step, we estimate the parameters for all stage models
$\{\model^{(\lstage)}\}_{\lstage=1}^{\nstage}$,
while fixing the stage assignments $\allstageas$.
Once the stage assignments are fixed, 
we can optimize each stage $\model^{(\lstage)}$ independently 
by maximizing \eq{\ref{eqn:overall_objective_d1}} and \eq{\ref{eqn:overall_objective_p}}.

We first focus on the descriptor 
$\{\dmean^{(\lstage)},\dprecision^{(\lstage)}\}$.
Since we solve the problem for each stage $\lstage$,
We can rewrite  in terms of each $\dprecision^{(\lstage)}$,
which is estimated so that it can maximize the following equation:

%
\begin{align}
\label{eqn:dmodel_update}
\argmax_{\dprecision \in \sds}~
n^{(\lstage)}(\log\det\dprecision^{(\lstage)} 
-\Tr(\empcov^{(\lstage)} \dprecision^{(\lstage)}))
- \lassopara ||\dprecision^{(\lstage)}||_{od,1},
\end{align}
\normalsize
which 
is equivalent to the graphical lasso problem \cite{glasso2008},
where $\dprecision^{(\lstage)}$ must be 
a symmetric positive-definite ($\sds$),
$n^{(\lstage)}$ is the number of observations assigned to stage $\lstage$,
and $\empcov^{(\lstage)}$ denotes the empirical covariance matrix 
of the observations that are assigned to the $\lstage$-th stage.
Since this is a convex optimization problem,
we use the alternating direction method of multipliers,
which efficiently converges on the globally optimal solution.
The parameter $\dmean^{(\lstage)}$ is 
derived from the empirical mean of the observations assigned to stage $\lstage$.

For the predictor 
$\{\pfunc^{(\lstage)},\bstd^{(\lstage)}\}$, 
%
we maximize \eq{\ref{eqn:overall_objective_p}} 
through maximum likelihood estimation.
We first estimate the link function $\pfunc^{(\lstage)}$ 
that minimizes residual errors,
i.e., min$\sum_{\stageas_{\lseq,\ltime}=\lstage} 
||\rul_{\lseq,\ltime} - \pfunc^{(\lstage)}(x_{\lseq,\ltime})||_2$.
The diffusion parameter $\bstd^{(\lstage)}$ is estimated as follows:
\hide{
Once we obtain the link function $\pfunc^{(\lstage)}$,
we estimate
an increments of the stage-specific progression process, 
$d\plevel^{(\lstage)}(\wtime) = 1/ \pfunc^{(\lstage)}(x_{\lseq,\ltime})$.
Since $\plevel(\wtime)$ has independent increments and is normally distributed,
the diffusion parameter can be \reminder{obtained as}:
}
\begin{align}
\label{eqn:pdiffusion_update}
\bstd^{(\lstage)}=
\biggl[
\frac{1}{n^{(\lstage)}}
\sum_{\stageas_{\lseq,\ltime}=\lstage} 
||\frac{1}{\rul_{\lseq,\ltime}}
- \mu^{(\lstage)}_{\rul}||_2
\biggr]
^{\frac{1}{2}}.
\end{align}

\subsection{Proof of Lemma~\ref{complexity_learning}}
\label{sec:app:complexity_learning}
\begin{customlemma}
The time complexity of the learning algorithm 
in \method is $O(\#iter \cdot \sum_{\lseq}\ntime_{\lseq})$.
\end{customlemma}
\begin{proof}
For each iteration, 
the learning algorithm first updates 
the descriptors and predictors for each stage.
This procedure takes $O(\#iter_d 
\cdot \nstage)$,
where $\#iter_d$ is the number of iterations needed to estimate the sparse precision matrix $\dprecision^{(\lstage)}$.
To update stage assignments $\allstageas$,
we need $O(\nstage^2\ntime_{\lseq})$ for each sequence according to \eq{\ref{eqn:assign_cost}}.
Therefore, the complexity of updating stage assignments is $O(\nstage^2 
\sum_{\lseq}\ntime_{\lseq})$.
Overall,
the algorithm repeats these two procedures until convergence.
It requires 
$O(\#iter \cdot (\nstage + \nstage^2 \sum_{\lseq}\ntime_{\lseq}))$,
where $\#iter$ is the total number of iterations required for convergence, including $\#iter_{d}$.
The number of stages $\nstage$ is constant and thus negligible. 
Thus, the complexity is $O(\#iter \cdot \sum_{\lseq}\ntime_{\lseq})$.

\end{proof}
Although the worst case complexity is dominated by the convergence, 
$\#iter$ is a small value for the total durations $\sum_{\lseq}\ntime_{\lseq}$,
as shown in Figures~\ref{fig:saclability}.
Thus, the computation time of the learning algorithm in \method 
scales linear on the total duration of the sequences.

\subsection{Proof of Lemma~\ref{complexity_streaming}}
\label{sec:app:complexity_streaming}
\begin{customlemma}
The time complexity of the streaming algorithm
in \method is 
$O((1 + \#iter) \cdot \nstage^2)$ 
amortized per time step.
\end{customlemma}
\begin{proof}
We consider a stream instance $X_\newlseq$.
For each time step $\ctime$, 
\method executes \textsc{AdaptivePredict}.
The algorithm first estimates 
the current stage $\stageas_{\newlseq,\ctime}$ for each observation $x_{\newlseq,\ctime}$
based on the updating rule (i.e., \eq{\ref{eqn:stage_assign_stream_help}}).
This update is formulated as a dynamic programming procedure over
stages, 
where the cost $\{\optcost{\lstage}{\ctime}\}_{\lstage=\lstage'}^{\nstage}$ 
is computed by considering all valid
transitions from previous stages $\lstage'$. 
Consequently, this step requires
$O(\nstage^2)$.
Then, it accesses the prediction model for the stage $\stageas_{\newlseq,\ctime}$ 
and predicts the time to event 
$p_{\newlseq,\ctime}(\rul)$.
This procedure takes $O(1)$.
Therefore, the total complexity of \textsc{AdaptivePredict} 
is $O(\nstage^2)$.

\textsc{OnlineModelUpdate} is 
executed only once per stream instance $X_{\newlseq}$,
after all $\ntime_\newlseq$ time steps 
have been processed (i.e., when $\ctime =\ntime_\newlseq$).
Similar to proof of Lemma~\ref{complexity_learning},
the iteration of \textsc{UpdateAssignments} and \textsc{UpdateStageModels}
requires $O(\#iter \cdot (\nstage + \nstage^2 \ntime_\newlseq))$.
Recall that $\#iter$ is the total number of iterations required for convergence.
Then, \textsc{UpdateStageModels-Online} uses Welford’s algorithm, which
takes $O(\nstage)$.
Finally, 
it evaluates the prediction accuracy of $\allmodel^{+}$,
requiring $O(\ntime_{\newlseq})$.
Hence, the total complexity of \textsc{OnlineModelUpdate} is
$O(\nstage + \ntime_\newlseq + \#iter \cdot (\nstage  +\nstage^2 \ntime_\newlseq))$

Over the entire stream instance $X_\newlseq$, \textsc{AdaptivePredict} is executed 
$\ntime_\newlseq$
times, resulting in a total cost of 
$O(\nstage^2\ntime_\newlseq)$.
Combining both parts, the total computational cost for $X_\newlseq$ is
$O(\nstage^2\ntime_\newlseq + \nstage + \ntime_\newlseq + \#iter \cdot (\nstage  +\nstage^2 \ntime_\newlseq))$.
Dividing the total cost by $\ntime_\newlseq$ yields an amortized per-step complexity of 
$O(\nstage^2 + \frac{\nstage}{\ntime_\newlseq} + 1 + \#iter \cdot (\frac{\nstage}{\ntime_\newlseq} + \nstage^2))$.
Since $\ntime_\newlseq$ is sufficiently large,
the per-step amortized time complexity is 
$O((1 + \#iter) \cdot \nstage^2)$.

\end{proof}

\section{Experiments}
\label{sec:app:exp}
\begin{figure*}[t]
    \centering
    \begin{tabular}{ccc}
     \hspace{-4em}
     \begin{minipage}{0.35\linewidth}
     \centering
    \vspace{0.5em}
    \hspace{-1.em}
    \includegraphics[width=0.82\linewidth]{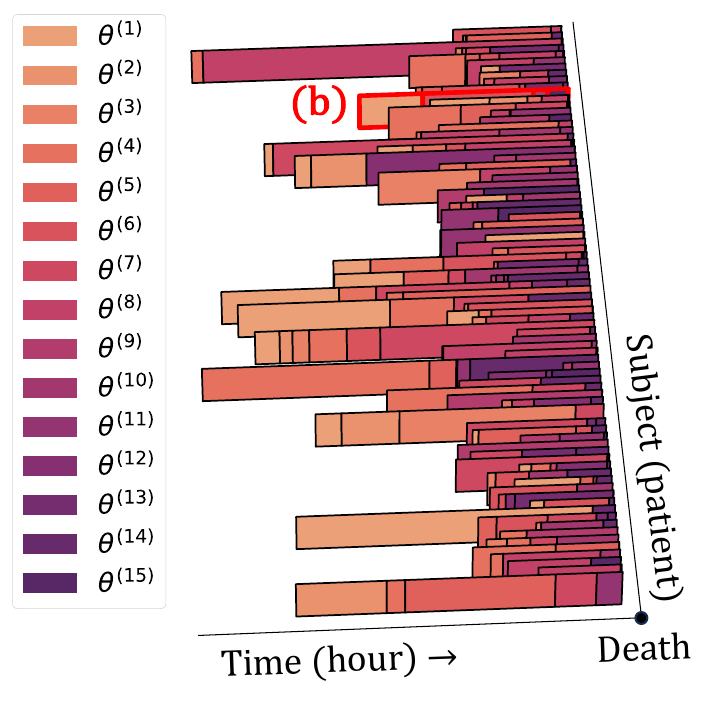}\\
     \vspace{-0.75em}
    \hspace{1em}
     (a) Stage assignments
     \end{minipage}
     &
    \begin{minipage}{0.52\linewidth}
    \vspace{2.5em}
    \hspace{-8.75em}
     \centering 
     \includegraphics[width=1\linewidth]{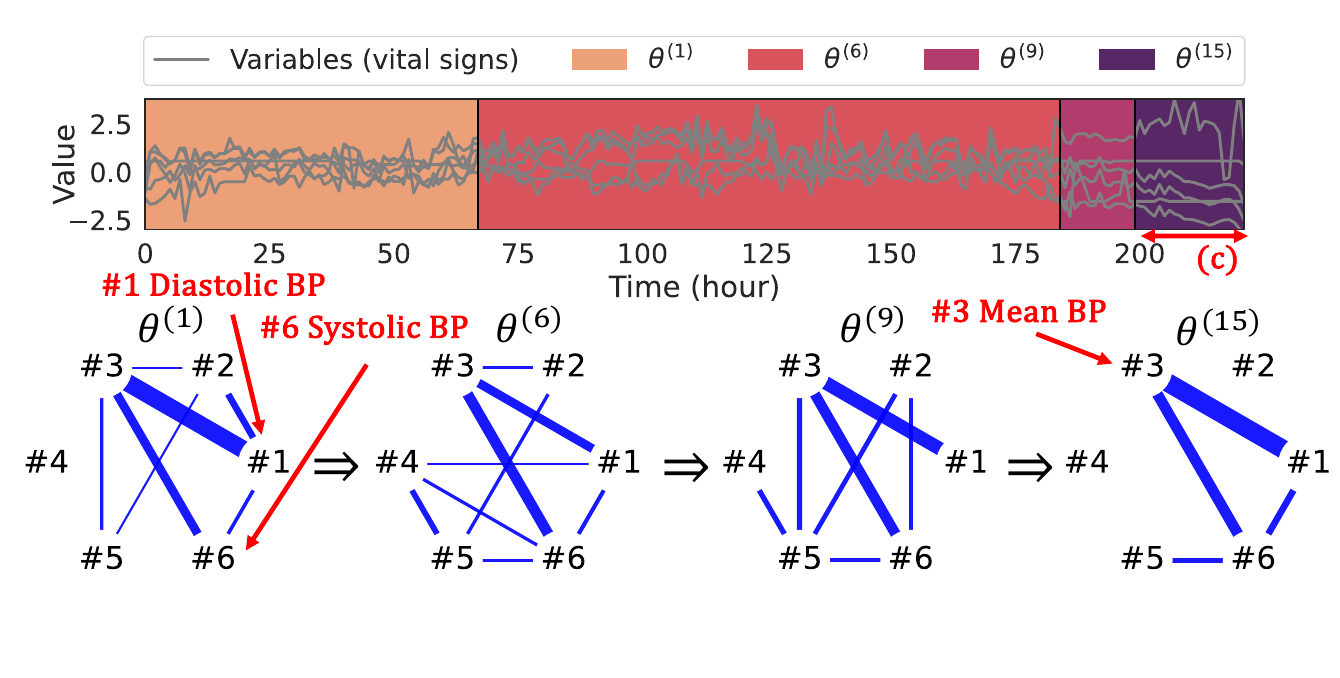}\\
     \vspace{-1.7em}
    \hspace{-8.2em}
     (b) Dynamic changes of \intera structure
    \end{minipage}
     &
     \hspace{-5.7em}
     \begin{minipage}{0.28\linewidth}
     \centering
    \vspace{4.65em}
     \includegraphics[width=1\linewidth]{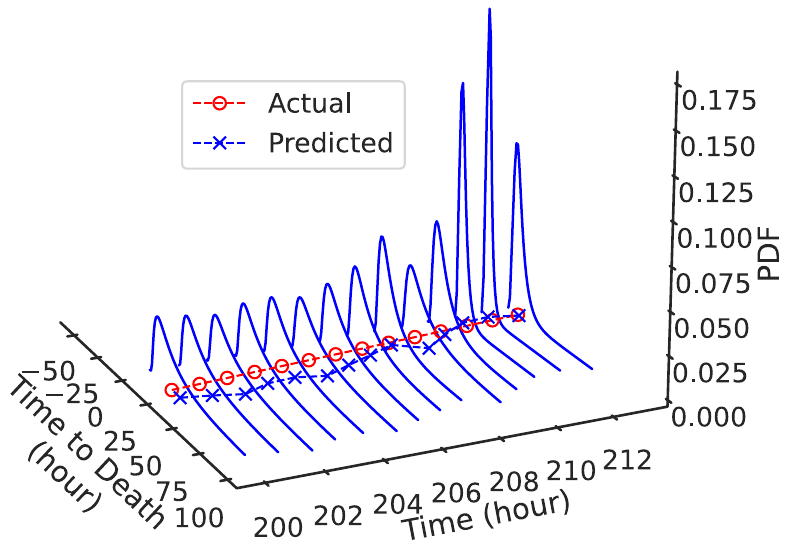}\\
    \hspace{-3em}
    (c) Time-to-event prediction
     \end{minipage}
    \end{tabular}
    \caption{
       Modeling power of \method on the \acutedata dataset.
       Given a set of sensor sequences,
       collected from $112$ patients followed until death,
       \method discovers (a) the stage assignments $\allstageas$,
       where a set of colored segments indicates the assignments of each stage.
       \method represents each stage by a stage model $\model^{(\lstage)}$,
       which can provide (b) the dynamic \intera structures between 
       the six vital signs
       and (c) the probability density function 
        of the time to death every hour.
    }
    \label{fig:eff}
    \vspace{-1em}
\end{figure*}


\subsection{Experimental Setup and Datasets}
\label{sec:app:expsetup:datasets}
We conducted our experiments on
an Intel Xeon Gold $6258$R @$2.70$GHz
with $512$GB of memory and running Linux.
We normalized the values so that each sequence 
had the same mean and variance (i.e., z-normalization).
%
MAPE and RMSPE are computed based on  
the percentage errors between the predicted event time $\hat{\rul}_{\newlseq,\ltime}$ and 
the actual event time $\rul_{\newlseq,\ltime}$:
\small
\begin{align}
MAPE=\frac{1}{\sum_{\newlseq=1}^{\newnseq}\ntime_{\newlseq}}
\sum_{\newlseq\ltime=1}^{\newnseq,\ntime_{\newlseq}}
\frac{|\hat{\rul}_{\newlseq,\ltime} - \rul_{\newlseq,\ltime}|}{\rul_{\newlseq,\ltime}},
\end{align}
\begin{align}
RMSPE=\sqrt{\frac{1}{\sum_{\newlseq=1}^{\newnseq}\ntime_{\newlseq}}
\sum_{\newlseq\ltime=1}^{\newnseq,\ntime_{\newlseq}}
\biggr(\frac{\hat{\rul}_{\newlseq,\ltime} - \rul_{\newlseq,\ltime}}{\rul_{\newlseq,\ltime}}\biggr)^2}.
\end{align}
\normalsize
Although 
\method can provide 
the probability distribution 
$p_{\newlseq,\ltime}(\hat{\rul})$,
we employ the mean of the distribution as the value $\hat{\rul}_{\newlseq,\ltime}$ for a fair comparison.

For 
The statistics of our datasets are provided in Table~\ref{table:datasets}.
Here, we briefly demonstrate how these datasets were prepared 
for our experiments.
\bit
\item 
\cmapss\cite{saxena2008damage}
\footnote{
https://data.nasa.gov/Aerospace/CMAPSS-Jet-Engine-Simulated-Data/ff5v-kuh6} 
is a public dataset for asset degradation modeling from NASA.
It includes the degradation data of turbofan jet engines simulated by C-MAPSS,
where each engine 
has different degrees of initial wear and manufacturing variation.
Sensor observations are collected at each cycle.
Since some sensor readings have constant outputs,
we use seven sensor measurements, 
2, 3, 4, 7, 11, 12, and 15,
following a previous study \cite{wang2008similarity},
i.e., 2: total temperature at an LPC outlet, 
3: total temperature at an HPC outlet, 
4: total temperature at an LPT outlet, 
7: total pressure at an HPC outlet,
11: static pressure at an HPC,
12: outlet psia phi Ratio of fuel flow to Ps30,
and 15: bypass ratio.

\item 
\mapm\footnote{https://github.com/Azure/AI-PredictiveMaintenance}
is a publicly available dataset that
consists of hourly averages of voltage, rotation, pressure, and vibration 
collected from $100$ machines for the year $2015$.
We use all the sensor measurements for
$98$ machines, where failures eventually occur.

\item 
\chronic is Medical Information Mart for Intensive Care (MIMIC) data.
We use MIMIC-\RomanNumeralCaps{3}
\cite{johnson2016mimic}, 
a large set of open electronic health records on
PhisioNet
\cite{goldberger2000physiobank}
that include vital signs, 
medications,
and laboratory measurements.
We follow the settings in \cite{harutyunyan2019multitask}, 
where each patient has ICU phenotypes, 
and observations are recorded every hour.
The phenotypes are grouped into three categories, 
chronic, acute, and mixed.
For the \chronic dataset, 
we studied $355$ patients whose chronic phenotypes
and employed six types of continuous sensor data, 
diastolic blood pressure, heart rate, mean blood pressure,
oxygen saturation, respiratory rate, and systolic blood pressure.
\item 
\acutedata is also MIMIC-\RomanNumeralCaps{3} data. 
We studied $112$ patients labeled with mutually conclusive acute phenotypes. 
\item 
\mixed is also derived using the MIMIC-\RomanNumeralCaps{3} data. 
We studied $521$ patients labeled with mutually conclusive mixed phenotypes. 
\eit

\subsection{Implementation \& Parameters}
\label{sec:app:expsetup:parameters}

We used the open-source implementation of 
DeepSurv, DeepHit, and Cox-Time in \cite{kvamme2019time}, 
and those of AC-TPC, TS2Vec, and PatchTST provided by the authors. 
The DNN-based models were optimized based on Adam. 
For the number of epochs, we employed widely used model checkpointing, where we monitor prediction performance on a validation set 
at each epoch and retain the model from the best-performing epoch to avoid overfitting.
For the learning rates, we searched multiple values suggested in the original papers. 
For DeepSurv, DeepHit and Cox-Time, 
we built a 2-layer fully-connected network with 32 nodes.
For AC-TPC, we set its label space $\mathcal{Y}=\mathbb{R}$ 
and learned 
the predictive cluster for regression tasks to predict $\rul_{\lseq,\ltime}$.
For clustering-based baselines, the number of clusters is set to $\nstage \in \{5, 10, 15, 20\}$.
For TS2Vec, 
we set the representation dimension at 128 
and trained a linear regression model with 
a $\ell_2$-norm penalty that 
took $x_{\newlseq,\ltime}$ as its input to directly predict the event time $\rul_{\newlseq,\ltime}$.

For all the methods, 
the input feature was set as a sliding window with the window size $\window$,
where we used $[x_{\lseq,\ltime-\window},\ldots, x_{\lseq,\ltime}]$ as the input $x_{\lseq,\ltime}$. 
The window size $\window$ was set at $10\%$ of the average sequence length for each dataset,
i.e., \cmapss: $\window=20$, \mapm: $\window=90$, \chronic:
$\window=30$, \acutedata: $\window=20$, and \mixed: $\window=30$.

\subsection{Results}
\label{sec:app:results}
\begin{table}[t]
\centering
\caption{Integrated Brier Score (IBS) comparison on the \#2 \mapm dataset (lower is better).}
\label{tab:ibs_results}
\scalebox{0.82}{
\begin{tabular}{l|c}
\toprule
Method & IBS \\
\midrule
\method  & \textbf{0.24483} \\
Cox-Time  & 0.51092 \\
DeepHit   & 0.51097 \\
DeepSurv  & 0.51114 \\
\bottomrule
\end{tabular}
}
\end{table}
\myparaitemize{Accuracy}
To further evaluate the accuracy of predicted event probabilities, 
we conduct additional experiments using standard survival analysis
metrics.
Specifically, we assess the predictive accuracy of our method and baseline survival models using the integrated Brier score (IBS), which measures the squared error between
the predicted survival probabilities and the observed event outcomes.

For each instance $\newlseq$ at time step $\ltime$,
the Brier score (BS) at horizon $\tau$ is defined as follows:
\begin{equation}
\mathrm{BS}(\newlseq,\ltime,\tau)
= \bigl( \mathbb{I}(T_{\newlseq} > \ltime + \tau) - \hat{S}_{\newlseq,\ltime}(\tau) \bigr)^2,
\end{equation}
where $\mathbb{I}(\cdot)$ is the indicator function and
$\hat{S}_{\newlseq,\ltime}(\tau)$ denotes the predicted survival probability.
The IBS is computed by averaging the BS over all instances, time steps, and prediction horizons:

\begin{equation}
\mathrm{IBS}
= \frac{1}{\sum_{\newlseq=1}^{\newnseq} \ntime_{\newlseq}}
\sum_{\newlseq=1}^{\newnseq}
\sum_{\ltime=1}^{\ntime_{\newlseq}}
\frac{1}{L}\sum_{\tau=1}^{L}
\mathrm{BS}(\newlseq,\ltime,\tau).
\end{equation}

Note that the survival function 
$\hat{S}_{\newlseq,\ltime}(\tau)$ in \method is obtained as the complement of the cumulative distribution function (CDF) of the predicted inverse Gaussian distribution. 
The CDF of the predicted inverse Gaussian distribution can be analytically derived from the estimated parameters of the stage-specific Wiener process, i.e., $\nu$ and $\sigma_B$. 
Denoting the CDF as $F_{\newlseq,\ltime}(\tau)$, the survival function is obtained as its complement,
$\hat{S}_{\newlseq,\ltime}(\tau) 
= 1 - F_{\newlseq,\ltime}(\tau)$, representing the probability that the event has not yet occurred by time $\tau$.

Table~\ref{tab:ibs_results} reports the IBS results on the \#2 \mapm dataset. \method consistently achieves lower IBS values compared to existing survival models, indicating superior predictive accuracy of event probabilities.

\begin{figure}[t]
    \centering
     \includegraphics[width=0.95\columnwidth]{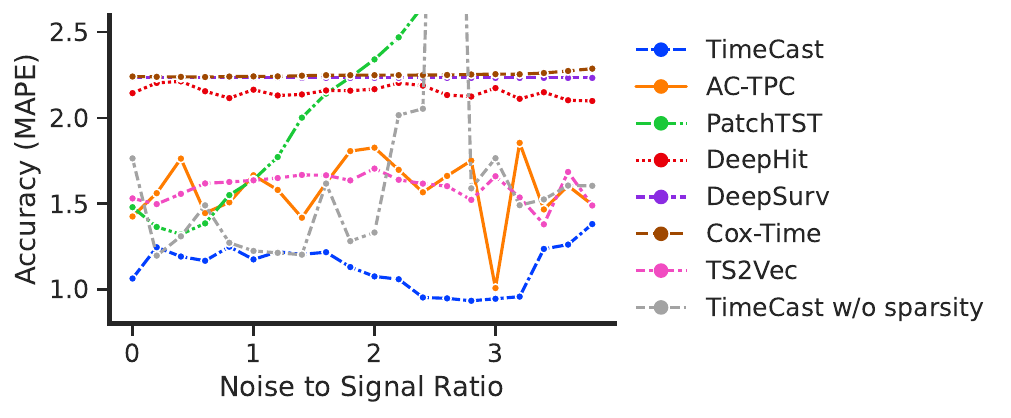}\\
    \caption{
    \method is consistently outperforms baseline methods.
    Imposing sparsity on the precision matrix makes its accuracy stable.
    }
    \label{fig:noise_vs}
\end{figure}
\myparaitemize{Noise Resilience of \method}
\fig{\ref{fig:noise_vs}} compares the mean absolute percentage error (MAPE) of the prediction methods 
when varying the noise to signal ratio on the \acutedata dataset, 
where \method w/o sparsity optimizes the stage models 
without imposing sparsity on the precision matrices. 
We incremented noise until the performance equal to the ratio, 
as follows \cite{tozzo2021statistical}.
\method is consistently more accurate than 
the baselines, 
while the performance of \method w/o sparsity 
deteriorated with the growth of the noise-to-signal ratio, 
suggesting our dependency network reduces the effect of noise.

\myparaitemize{Real-World Effectiveness}
Figure~\ref{fig:eff} shows our mining result
for the \acutedata dataset,
which consists of continuous patient monitoring data recorded in ICUs,
where six vital signs were collected every hour from $112$ patients.
All the patients were followed 
in acute care until their deaths, 
which resulted from clinically critical events, 
such as respiratory failure or sepsis.

Figure~\ref{fig:eff}~(a) shows 
the discovered stage assignments $\allstageas$,
which identify distinct time-series patterns and their shifting points.
Although 
\reminder{patient's conditions vary over time 
depending on the clinical interventions and the potential risk of diseases,
this representation}
allows us to find
similar patient behavior.
Figure~\ref{fig:eff}~(b) shows 
the stage assignments and dynamic changes in \intera structures 
(i.e., $\dprecision^{(\lstage)}$) for a patient with respiratory failure.
Here, we observe that 
sensor \#3 (mean blood pressure) is consistently connected to
\#1 (diastolic blood pressure) and \#6 (systolic blood pressure) over all stages.
This means that variations in diastolic and systolic blood pressure
depend on mean blood pressure regardless of the risk of respiratory failure.
Figure~\ref{fig:eff}~(c) shows a snapshot of time-to-event prediction in the patient.
\method continuously estimates the current stage $\stageas_{\newlseq,\ctime}$
for the observation $x_{\newlseq,\ctime}$
and adaptively predicts the event time $p(\rul_{\newlseq,\ctime})$,
employing the stage model $\model^{(\stageas_{\newlseq,\ctime})}$.
\end{document}